\documentclass[twoside,11pt]{article}

\usepackage{shortcuts_mm}
\usepackage{jmlr2e}
\usepackage[group-separator={,},group-minimum-digits=4]{siunitx}

\usepackage[nameinlink]{cleveref}
\newtheorem{assumption}[theorem]{Assumption}
\Crefname{proposition}{Proposition}{Propositions}
\crefname{problem}{Problem}{Problems}
\Crefname{lemma}{Lemma}{Lemmas}

\usepackage{aliascnt}
\newaliascnt{problem}{equation}
\aliascntresetthe{problem}
\creflabelformat{problem}{#2\textup{(#1)}#3}

\makeatletter

\def\endproblem{\eqno \hbox{\@eqnnum}$$\@ignoretrue}
\makeatother
\DeclareMathOperator{\relint}{relint}

\graphicspath{{./images/}}
\usepackage[export]{adjustbox}

\usepackage{lastpage}
\jmlrheading{21}{2020}{1-\pageref{LastPage}}{7/19; Revised
8/20}{10/20}{19-587}{Mathurin Massias, Samuel Vaiter, Alexandre Gramfort and Joseph Salmon}
\ShortHeadings{Dual Extrapolation for Sparse GLMs}{Massias, Vaiter, Gramfort and Salmon}

\firstpageno{1}

\usepackage[normalem]{ulem}
\begin{document}

\title{Dual Extrapolation for Sparse Generalized Linear Models}

\author{\name Mathurin Massias \email mathurin.massias@inria.fr \\
       \addr Universit\'e Paris-Saclay, Inria, CEA, Palaiseau, France\\
       \AND
       \name Samuel Vaiter  \email samuel.vaiter@u-bourgogne.fr \\
       \addr  CNRS \& Institut de Math\'ematiques de Bourgogne, 21078, Dijon, France\\
       \AND
       \name Alexandre Gramfort  \email alexandre.gramfort@inria.fr \\
       \addr Universit\'e Paris-Saclay, Inria, CEA, Palaiseau, France\\
       \AND
       \name Joseph Salmon  \email joseph.salmon@umontpellier.fr \\
       \addr IMAG, Univ Montpellier, CNRS, 34095, Montpellier, France \\
       }

\editor{David Wipf}

\maketitle
\begin{abstract}%

Generalized Linear Models (GLM) form a wide class of regression and classification models, where prediction is a function of a linear combination of the input variables.
For statistical inference in high dimension, sparsity inducing regularizations have proven to be useful while offering statistical guarantees.
However, solving the resulting optimization problems can be challenging: even for popular iterative algorithms such as coordinate descent, one needs to loop over a large number of variables.
To mitigate this, techniques known as \emph{screening rules} and \emph{working sets} diminish the size of the optimization problem at hand, either by progressively removing variables, or by solving a growing sequence of smaller problems.
For both techniques, significant variables are identified thanks to convex duality arguments.
In this paper, we show that the dual iterates of a GLM exhibit a Vector AutoRegressive (VAR) behavior after sign identification, when the primal problem is solved with proximal gradient descent or cyclic coordinate descent.
Exploiting this regularity, one can construct dual points that offer tighter certificates of optimality, enhancing the performance of screening rules and working set algorithms.
 \end{abstract}


\section{Introduction}
\label{sec:introduction}
Since the introduction of the Lasso \citep{Tibshirani96}, similar to the Basis Pursuit denoising \citep{Chen_Donoho95} in signal processing, sparsity inducing penalties have had a tremendous impact on machine learning \citep{Bach_Jenatton_Mairal_Obozinski12}.
They have been applied to a variety of statistical estimators, both for regression and classification tasks: sparse logistic regression~\citep{Koh_Kim_Boyd07}, Group Lasso \citep{Yuan_Lin06}, Sparse Group Lasso \citep{Simon_Friedman_Hastie_Tibshirani12}, multitask Lasso \citep{Obozinski_Taskar_Jordan10}, Square-Root Lasso \citep{Belloni_Chernozhukov_Wang11}.
All of these estimators fall under the framework of Generalized Linear Models~\citep{Mccullagh_Nelder1989}, where the output is assumed to follow an exponential family distribution whose mean is a linear combination of the input variables.
The key property of $\ell_1$ regularization is that it allows to jointly perform feature selection and prediction, which is particularly useful in high dimensional settings.
Indeed, it can drastically reduce the number of variables needed for prediction, thus improving model interpretability and computation time for prediction.
Amongst the algorithms proposed to solve these, coordinate descent\footnote{throughout the paper, this means \emph{cyclic and proximal} coordinate descent} \citep{Tseng01,Friedman_Hastie_Hofling_Tibshirani07} is the most popular in machine learning scenarios
\citep{Fan_Chang_Hsieh_Wang_Lin08,Friedman_Hastie_Tibshirani10,Richtarik_Takac14,Fercoq_Richtarik15,Perekrestenko_Cevher_Jaggi17,Karimireddy_Koloskova_Stich_Jaggi18}.
It consists in updating the vector of parameters one coefficient at a time,  looping over all the predictors until convergence.

Since only a fraction of the coefficients are non-zero in the optimal parameter vector, a recurring idea to speed up solvers is to limit the size of the optimization problem by ignoring features which are not included in the solution.
To do so, two approaches can be distinguished:
\begin{itemize}
     \item \emph{screening rules}, introduced by \citet{ElGhaoui_Viallon_Rabbani12} and later developed by  \cite{Ogawa_Suzuki_Takeuchi13,Wang_Wonka_Ye12,Xiang_Wang_Ramadge14,Bonnefoy_Emiya_Ralaivola_Gribonval14,Fercoq_Gramfort_Salmon15,Ndiaye_Fercoq_Gramfort_Salmon16,Ndiaye_Fercoq_Gramfort_Salmon16b}, progressively remove features from the problems in a backward approach,
     \item \emph{working sets} techniques \citep{Fan_Lv2008,Roth_Fischer08,Kowalski_Weiss_Gramfort_Anthoine11,Tibshirani_Bien_Friedman_Hastie_Simon_Tibshirani12,Johnson_Guestrin15} solve a sequence of smaller problems restricted to a growing number of features.
 \end{itemize}
One common idea between the current state-of-art methods for screening (Gap Safe rules \citealt{Fercoq_Gramfort_Salmon15,Ndiaye_Fercoq_Gramfort_Salmon16b}) and working sets (\blitz, \citealt{Johnson_Guestrin15,Johnson_Guestrin18}) is to rely heavily on a dual point to identify useful features.
The quality of such a dual point for the dual problem is critical here as it has a direct impact on performance.
However, although a lot of attention has been devoted to creating a sequence of primal iterates that converge fast to the optimum \citep{Fercoq_Richtarik15}, the construction of dual iterates has not been scrutinized, and the standard approach to obtain dual iterates from primal ones \citep{Mairal}, although converging, is crude.

In this paper, we propose a principled way to construct a sequence of dual points that converges faster than the standard approach proposed by \citet{Mairal}.
Based on an extrapolation procedure inspired by \citet{Scieur_Bach_Daspremont16}, it comes with no significant extra computational costs,
while retaining convergence guarantees of the standard approach.
This construction was first introduced for non-smooth optimization by \citet{Massias_Gramfort_Salmon18} for the Lasso case only, while we generalize it here to any Generalized Linear Model (GLM).
More precisely, we properly define, quantify and prove the asymptotic Vector AutoRegressive (VAR) behavior of dual iterates for sparse GLMs solved with proximal gradient descent or cyclic coordinate descent.
The resulting new construction:
\begin{itemize}
    \item provides a tighter control of optimality through duality gap evaluation,
    \item improves the performance of Gap safe rules,
    \item is easy to implement and combine with other solvers.
\end{itemize}

The article proceeds as follows.
We introduce the framework of $\ell_1$-regularized GLMs and duality in \Cref{sec:notation_and_framework}.
As a seminal example, we present our results on dual iterates regularity and dual extrapolation for the Lasso in \Cref{sec:lasso_case}.
We generalize it to a variety of problems in \Cref{sec:others_models,sec:screening_working_sets}.
Results of \Cref{sec:experiments} demonstrate a systematic improvement in computing time when dual extrapolation is used together with Gap Safe rules or working set policies.

\paragraph{Notation}
For any integer $d \in \bbN$, we denote by $[d]$ the set $\{1, \dots, d\}$.
The design matrix $X \in \bbR^{n \times p}$ is composed of observations $\mathbf{x}_i \in \bbR^p$ stored row-wise, and whose $j$-th column is $x_j \in \bbR^n$; the vector $y \in \bbR^n$ (resp. $\{-1, 1\}^n$) is the response vector for regression (resp. binary classification).
The support of $\beta \in \bbR^p$ is $\cS(\beta) = \condsetin{j\in [p]}{\beta_j \neq 0}$, of cardinality $\normin{\beta}_0$.
For $\cW \subset [p]$, $\beta_\cW$ and $X_\cW$ are $\beta$ and $X$ restricted to features in $\cW$.
As much as possible, exponents between parenthesis (\eg $\beta^{(t)}$) denote iterates and subscripts (\eg $\beta_j$) denote vector entries or matrix columns.
The sign function is $\sign: x \mapsto x / \absin{x}$ with the convention $0/0 = 0$.
The sigmoid function is $\sigma: x \mapsto 1 / (1 + e^{-x})$.
The soft-thresholding of $x$ at level $\nu$ is $\ST(x, \nu) = \sign(x) \cdot \max(0, \abs{x} - \nu)$.
Applied to vectors, $\sign$, $\sigma$ and $\ST(\cdot, \nu)$ (for $\nu \in \bbR_+$) act element-wise.
Element-wise product between vectors of same length is denoted by $\odot$.
The vector of size $n$ whose entries are all equal to 0 (resp. 1) is denoted by $\mathbf{0}_n$ (resp. $\mathbf{1}_n$).
On square matrices, $\normin{\cdot}_2$ is the spectral norm (and the standard Euclidean norm for vectors reads $\norm{\cdot}$); $\normin{\cdot}_1$ is the $\ell_1$-norm.
For a symmetric positive definite matrix $H$, $\langle x , y \rangle_H = x^\top H y$ is the $H$-weighted inner product, whose associated norm is denoted $\normin{\cdot}_H$.
We extend the small-$o$ notation to vector valued functions in the following way: for $f: \bbR^n \to \bbR^n$ and $g: \bbR^n \to \bbR^n$, $f = o(g)$ if and only if $\normin{f} = o(\normin{g})$, \ie $\normin{f} / \normin{g}$ tends to 0 when $\normin{g}$ tends to 0.
For a convex and proper function $f : \bbR^n \to \bbR \cup \{\infty\}$, its Fenchel-Legendre conjugate $f^*: \bbR^n \to \bbR \cup \{\infty\}$ is defined by $f^*(u) = \sup_{x\in \bbR^n} u^\top x - f(x)$,
and its subdifferential at $x \in \bbR^n$ is $\partial f(x) = \condsetin{u \in \bbR^n}{\forall y \in \bbR^n, f(y) \geq f(x) + u ^ \top (y - x)}$.
\section{GLMs, Vector AutoRegressive sequences and sign identification}
\label{sec:notation_and_framework}

We first introduce the class of optimization problems we consider.
\begin{definition}[Sparse Generalized Linear Model]
    We call \emph{Sparse Generalized Linear Model} the following optimization problem:
  \begin{problem}\label{eq:sparse_glm}
    \betaopt \in \argmin_{\beta \in \bbR^p} \underbrace{\sum_{i = 1}^n f_i( \beta^\top \mathbf{x}_i) + \lambda \normin{\beta}_1}_{\cP(\beta)}\enspace,
  \end{problem}
where all $f_i$ are convex\footnote{by that we mean close, convex and proper following the framework of \cite{Bauschke_Combettes11}.}, differentiable functions with $1/ \gamma$-Lipschitz gradients.
The parameter $\lambda$ is a non-negative scalar, controlling the trade-off between data fidelity and regularization.
\end{definition}

Two popular instances of \Cref{eq:sparse_glm} are the Lasso ($f_i(t) = \tfrac{1}{2}(y_i - t)^2$, $\gamma=1$) and Sparse Logistic regression
($f_i(t) = \log (1 + \exp(- y_i  t))$, $\gamma=4$); our naming is an abuse of language since for some choice of $f_i$'s, \eg an Huber loss, there is no underlying statistical GLM.

A more complex regularizer could be used in \Cref{eq:sparse_glm}, to handle group penalties for example.
For the sake of clarity we rather remain specific, and generalize to other penalties when needed in \Cref{sub:multitask_lasso}.

\begin{proposition}[Duality for sparse GLMs]\label{prop:dual_sparse_glm}
    A dual formulation of \Cref{eq:sparse_glm} reads:
    \begin{problem}\label{pb:dual_sparse_glm}
        \thetaopt = \argmax_{\theta \in \Delta_X} \underbrace{ \left(- \sum_{i=1}^{n} f_i^*(-\lambda \theta_i) \right)}_{\cD(\theta)} \enspace,
    \end{problem}
    where $\Delta_X = \condsetin{\theta \in \bbR^n}{\normin{X^\top \theta}_\infty \leq 1}$.
    The dual solution $\thetaopt$ is unique because the $f_i^*$'s are $\gamma$-strongly convex \citep[Thm 4.2.1]{Hiriart-Urruty_Lemarechal93b} and the KKT conditions read:
    \begin{align}\label{eq:KKT}
        \forall i \in [n],\quad \quad  & \, \thetaopt_i = -  f_i' (\betaopt^\top \mathbf{x}_i) / \lambda &\textbf{(link equation)}\\
        \forall j \in [p], \quad \quad & \, x_j^\top \thetaopt \in \partial \absin{\cdot} (\betaopt_j) &\textbf{(subdifferential inclusion)} \label{eq:subdiff_inclusion}
    \end{align}
    If for $u \in \bbR^n$ we write $F(u) \eqdef \sum_{i=1}^n f_i (u_i)$, the link equation reads $\thetaopt = - \nabla F(X\betaopt) / \lambda$.

    For any $(\beta, \theta) \in \bbR^p \times \Delta_X$, one has $\cD(\theta) \leq \cP(\beta)$, and $\cD(\thetaopt) = \cP(\betaopt)$.
    The duality gap $\cP(\beta) - \cD(\theta)$ can thus be used as an upper bound for the sub-optimality of a primal vector $\beta$: for any $\epsilon > 0$, any $\beta \in \bbR^p$, and any feasible  $\theta \in \Delta_X$:
    \begin{equation}\label{eq:stopping_criterion_gap}
        \cP(\beta) - \cD(\theta) \leq \epsilon \Rightarrow \cP(\beta) - \cP(\betaopt) \leq \epsilon\enspace.
    \end{equation}
    These results holds because Slater's condition is met: \Cref{eq:sparse_glm} is unconstrained and the objective function has domain $\bbR^p$ \cite[\S 5.2.3]{Boyd_Vandenberghe04}, therefore strong duality holds.
\end{proposition}
\begin{remark}\label{rem:dual_gap_stop_cond}
\Cref{eq:stopping_criterion_gap} shows that even though $\betaopt$ is unknown in practice and the sub-optimality gap cannot be evaluated, creating a dual feasible point $\theta \in \Delta_X$ allows to compute an upper bound which can be used as a tractable stopping criterion.
\end{remark}

In high dimension, solvers such as proximal gradient descent (PG) and coordinate descent (CD) are slowed down due to the large number of features.
However, by design of the $\ell_1$ penalty, $\betaopt$ is known to be sparse, especially for large values of $\lambda$.
Thus, a key idea to speed up these solvers is to identify the support of $\betaopt$ so that features outside of it can be safely ignored: this leads to a smaller problem that is faster to solve.
Removing features when it is guaranteed that they are not in the support of the solution is at the heart of the so-called \emph{Gap Safe Screening rules} \citep{Fercoq_Gramfort_Salmon15, Ndiaye_Fercoq_Gramfort_Salmon16b}.

\begin{proposition}[{Gap Safe Screening rule, \citep[Thm. 6]{Ndiaye_Fercoq_Gramfort_Salmon16b}}]\label{prop:gap_safe_rule}
    The Gap Safe screening rule for \Cref{eq:sparse_glm} reads:
    \begin{align}\label{eq:gap_safe_rule}
       \forall j \in [p], \forall \beta \in \bbR^p, \forall \theta \in \Delta_X, \, \frac{1 - \absin{x_j^\top \theta}}{\normin{x_j}} >\sqrt{\tfrac{2}{\gamma \lambda^2}(\cP(\beta) - \cD(\theta))} \implies \betaopt_j = 0 \enspace .
    \end{align}
\end{proposition}
Therefore, while running an iterative solver, the criterion~\eqref{eq:gap_safe_rule} can be tested periodically for all features $j$, and the features guaranteed to be inactive at optimum can be ignored.\footnote{\citet[Thm. 5.1]{Johnson_Guestrin18} improved the RHS in \eqref{eq:gap_safe_rule} by a factor $\sqrt{2}$. In our experiments, it did not lead to a noticeable speed-up as the bulk of the computation is spent on iterations before the screening rule discards variables, and the $\sqrt{2}$ factor is not large enough to make this happen much earlier}

\Cref{eq:stopping_criterion_gap,eq:gap_safe_rule} do not require a specific choice of $\theta$, provided it is in $\Delta_X$. It is up to the user and so far it has not attracted much attention in the literature.
Thanks to the link equation $\thetaopt = - \nabla F(X\betaopt) / \lambda$, a natural way to construct a dual feasible point $\theta^{(t)} \in \Delta_X$ at iteration $t$, when only a primal vector $\beta^{(t)}$ is available, is:
\begin{equation}\label{eq:theta_res}
    \thetaresiduals^{(t)} \eqdef - \nabla F(X\beta^{(t)}) / \max(\lambda, \normin{X^\top \nabla F(X\beta^{(t)})}_\infty)\enspace.
\end{equation}

This was coined \emph{residuals rescaling} \citep{Mairal} following the terminology used for the Lasso case where $- \nabla F(X\beta)$ is equal to the residuals, $ y - X \beta$.

To improve the control of sub-optimality and identification of useful features, the aim of our proposed \emph{dual extrapolation} is to obtain a better dual point (\ie closer to the optimum $\thetaopt$).
The idea is to do it at a low computational cost by exploiting the structure of the sequence of dual iterates $(X\beta^{(t)})_{t \in \bbN}$; we explain what is this ``structure'', and how to exploit it, in the following.
\begin{definition}[Vector AutoRegressive sequence]\label{def:var}
    We say that $(r^{(t)})_{t \in \bbN} \in (\bbR^n)^\bbN$ is a Vector AutoRegressive (VAR) sequence (of order 1)
    if there exists $A \in \bbR^{n \times n}$ and $b \in \bbR^n$ such that for $t \in \bbN$:
        \begin{equation}
            r^{(t + 1)} = A r^{(t)} + b \enspace.
        \end{equation}
        We also say that the sequence $(r^{(t)})_{t \in \bbN}$, converging to $\hat r$, is an asymptotic VAR sequence if there exist $A \in \bbR^{n \times n}$
    and $b \in \bbR^n$ such that for $t \in \bbN$:
        \begin{equation}
            r^{(t + 1)} - A r^{(t)} - b = o(r^{(t)} - \hat r) \enspace.
        \end{equation}
\end{definition}
\begin{proposition}[Extrapolation for VAR sequences {\citep[Thm 3.2.2]{Scieur}}]\label{prop:extrapolation_var}
     Let $(r^{(t)})_{t \in \bbN}$ be a VAR sequence in $\bbR^n$, satisfying $r^{(t + 1)} = A r^{(t)} + b$ with $A \in \bbR^{n \times n}$ a symmetric positive definite matrix such that $\normin{A}_2 < 1$, $b \in \bbR^n$ and $K < n$.
     Assume that for $t \geq K$, the family $\{ r^{(t  - K)} - r^{(t - K + 1)}, \ldots, r^{(t - 1)} - r^{(t)} \}$ is linearly independent and define
     \begin{align}\label{eq:r_extr_def-SAM}
         U^{(t)}
            &\eqdef [r^{(t  - K)} - r^{(t - K + 1)}, \ldots, r^{(t - 1)} - r^{(t)}] \in \bbR^{n \times K}\enspace, \\
         (c_1, \dots, c_{K})
            &\eqdef \frac{(U^{(t)}{}^\top U^{(t)}){}^{-1} \mathbf{1}_K}
                    {\mathbf{1}_K^\top (U^{(t)}{}^\top U^{(t)}){}^{-1} \mathbf{1}_K} \in \bbR^K \enspace, \\
         r_{\extr}
            &\eqdef \sum_{k = 1}^{K} c_k r^{(t - K - 1 + k)} \in \bbR^n \enspace.
     \end{align}
     Then, $r_{\extr}$ satisfies
     \begin{equation}\label{eq:r_extr_bound-SAM}
         \norm{A r_{\extr} - b - r_{\extr}} \leq \cO(\rho^K) \enspace,
     \end{equation}
     where $\rho = \frac{1 - \sqrt{1 - \normin{A}_2}}{1 + \sqrt{1 - \normin{A}_2}} < 1$.
\end{proposition}
The justification for this extrapolation procedure is the following: since $\normin{A}_2 < 1$, $(r^{(t)})_{t \in \bbN}$ converges, say to $\hat r$.
For $t \in \bbN$, we have $r^{(t + 1)} - \hat r = A (r^{(t)} - \hat r)$.
Let $(a_0, \dots, a_n) \in \bbR^{n + 1}$ be the coefficients of $A$'s characteristic polynomial.
By Cayley-Hamilton's theorem, $\sum_{k=0}^n a_k A^k = 0$.
Given that $\normin{A}_2 < 1$, $1$ is not an eigenvalue of $A$ and $\sum_{k=0}^n  a_k \neq 0$, so we can normalize these coefficients to have $\sum_{k=0}^n  a_k = 1$.
For $t \geq n$, we have:
\begin{align}
    \sum_{k=0}^n a_k \left(r^{(t - n + k)} - \hat r \right) &= \left( \sum_{k=0}^n a_k A^k \right) (r^{(t - n)} - \hat r) = 0 \enspace,  \\
    \mathrm{and ~ so} \quad \quad  \sum_{k=0}^n a_k r^{(t - n + k)} &= \sum_{k=0}^n a_k \hat r = \hat r \enspace.
\end{align}
Hence, $\hat r \in \Span(r^{(t - n)}, \dots, r^{(t)})$.

Therefore, it is natural to seek to approximate $\hat r$ as an affine combination of the $(n + 1)$ last iterates $(r^{(t - n)}, \dots, r^{(t)})$.
Using $(n + 1)$ iterates might be costly for large $n$, so one might rather consider only a smaller number $K$, \ie find $(c_1, \dots, c_{K}) \in \bbR^{K}$ such that $\sum_{k=1}^{K} c_k r^{(t - K -1 + k)}$ approximates $\hat r$.
Since $\hat r$ is a fixed point of $r \mapsto Ar + b$, $\sum_{k=1}^{K} c_k r^{(t - K - 1 + k)}$ should be one too.
Under the normalizing condition $\sum_{k=1}^{K} c_k =1$, this means that
\begin{align}\label{eq:fixed_point_ck}
     \sum_{k=1}^{K} c_k r^{(t - K - 1 + k)} - A \sum_{k=1}^{K}  c_k r^{(t - K - 1 + k)} - b
     &= \sum_{k=1}^{K} c_k r^{(t - K - 1 + k)} - \sum_{k=1}^{K} c_k \left(r^{(t - K + k)} - b\right) - b %
     \nonumber\\
     &= \sum_{k=1}^{K} c_k \left(r^{(t - K - 1 + k)} - r^{(t - K + k )}\right) %
 \end{align}
should be as close to $\mathbf{0}_n$ as possible; this leads to solving:
\begin{equation}\label{eq:minimization_c}
  \hat{c} = \argmin_{\substack{c\in \bbR^{K}\\ c^\top \mathbf{1}_{K} = 1}}
            \norm{\sum_{k=1}^{K} c_k \left(r^{(t - K + k)} - r^{(t - K - 1 + k)}\right)} \enspace,
\end{equation}
which admits a closed-form solution if $U^{(t)} \eqdef [r^{(t - K + 1)} - r^{(t  - K)}  , \ldots, r^{(t)} - r^{(t - 1)}] \in \bbR^{n \times K}$ has full column rank \citep[Lemma 2.4]{Scieur_Bach_Daspremont16}:
\begin{equation}\label{eq:closed_form_ck}
    \hat{c} = \frac{(U^{(t)}{}^\top U^{(t)}){}^{-1} \mathbf{1}_K}
             {\mathbf{1}_K^\top (U^{(t)}{}^\top U^{(t)}){}^{-1} \mathbf{1}_K} \enspace.
\end{equation}
In practice, the next proposition shows that when $U^{(t)}$ does not have full column rank, it is theoretically sound to use a lower value for the number of extrapolation coefficients $K$.
\begin{proposition}
  \label{prop:U_invertible}
  If $U^{(t)}{}^\top U^{(t)}$ is not invertible, then $\hat r \in \Span(r^{(t  - 1)}, \ldots, r^{(t - K)})$.
\end{proposition}
\begin{proof}
  Let $x \in \bbR^K \setminus \{\mathbf{0}_K\}$ be such that $U^{(t)}{}^\top U^{(t)} x = \mathbf{0}_K$,
  with $ x_K \neq 0$ (the proof is similar if $x_K = 0, x_{K -1} \neq 0$, etc.).
  Then $ U^{(t)} x = \sum_{k=1}^K x_k (r^{(t - K + k)} - r^{(t  - K + k - 1)}) =  0$ and, setting $x_{0} \eqdef 0$, $ r^{(t)} = \tfrac{1}{x_K}  \sum_{k=1}^K (x_{k} - x_{k - 1}) r^{(t - K + k - 1)}  \in \Span(r^{(t  - 1)}, \ldots, r^{(t - K)})$.
  Since $ \tfrac{1}{x_K} \sum_{k=k}^K (x_{k} - x_{k - 1}) = 1$,
  it follows that
  \begin{align}
  r^{(t + 1)} & = A r^{(t)} + b \nonumber \\
              & = \tfrac{1}{x_K} \sum_{k=1}^K (x_{k} - x_{k - 1}) (A r^{(t - K + k - 1)} + b) \nonumber \\
              & = \tfrac{1}{x_K} \sum_{k=1}^K (x_{k} - x_{k + 1}) r^{(t - K + k)} \in \Span(r^{(t  - 1)}, \ldots, r^{(t - K)}) \enspace ,
  \end{align}
  and subsequently $r^{(s)} \in \Span(r^{(t  - 1)}, \ldots, r^{(t - K)})$ for all $s \geq t$.
  By going to the limit, $\hat r \in \Span(r^{(t  - 1)}, \ldots, r^{(t - K)})$.
\end{proof}

Finally, we state the results on sign identification, which implies support identification. For these results, which connect sparse GLMs to VAR sequences and extrapolation, we need to make the following assumption.

\begin{assumption}\label{assum:uniqueness}
  \Cref{eq:sparse_glm} is non degenerate: $- \nabla f(\hat\beta) / \lambda \in \relint \partial \normin{\cdot}_1$, where $\relint$ denotes the relative interior and $f(\beta) = F(X\beta)$.
\end{assumption}
This non-degeneracy condition is frequently used in works on support identification \citep{fuchs2004sparse,hare2007identifying,candes2013simple,vaiter2015model}.
Using it, we can extend results by \citet{Hale_Yin_Zhang2008} about sign identification from proximal gradient to coordinate descent.
\begin{theorem}[Sign identification for proximal gradient and coordinate descent]\label{thm:sign_id}
    Let \Cref{assum:uniqueness} hold.
    Let $(\beta^{(t)})_{t \in \bbN}$ be the sequence of iterates converging to $\betaopt$ and produced by proximal gradient descent or coordinate descent when solving \Cref{eq:sparse_glm} (reminded in lines \ref{algoline:update_pg} and \ref{algoline:update_cd} of
    \Cref{alg:CD_ISTA_dual_extrapolation}).

    There exists $T \in \bbN$ such that: $\forall j \in [p], t \geq T \implies \sign(\beta^{(t)}_j) = \sign(\betaopt_j)$.
    The smallest epoch $T$ for which this holds is when \emph{sign identification} is achieved.
\end{theorem}
\begin{proof}
For lighter notation in this proof, we denote $l_j = \normin{x_j}^2 / \gamma$ and
$h_j(\beta) = \beta_j - \frac{1}{l_j} x_j^\top \nabla F (X\beta)$.
For $j \in [p]$, the subdifferential inclusion \eqref{eq:subdiff_inclusion} reads:
\begin{equation}\label{eq:optimality_subdiff}
  - \frac{x_j^\top \nabla F(X\betaopt)}{\lambda} \in
    \begin{cases}
    \{1\} \enspace , & \text{ if } \betaopt_j > 0 \enspace,\\
    \{-1\} \enspace, & \text{ if } \betaopt_j < 0 \enspace,\\
    [-1,1] \enspace, & \text{ if } \betaopt_j = 0 \enspace.
    \end{cases}
\end{equation}
Motivated by these conditions, the \emph{equicorrelation set} introduced by \citet{Tibshirani13} is:
\begin{equation}
  E \eqdef \condsetin{j \in [p]}{  |x_j^\top \nabla F(X\betaopt)| = \lambda} = \condsetin{j \in [p]}{  |x_j^\top \thetaopt| = 1} \enspace.
\end{equation}
We introduce the \emph{saturation gap} associated to \Cref{eq:sparse_glm}:
\begin{align}\label{eq:def_saturation_gap}
    \hat{\delta}
    & \eqdef
     \min \left\{ \frac{\lambda}{l_j} \left(1-\frac{| x_j^\top \nabla F(X\betaopt)|}{\lambda}\right): j \notin E \right\} =
     \min \left\{ \frac{\lambda}{l_j} \left(1-| x_j^\top \thetaopt|\right): j \notin E \right\} > 0
    \enspace.
\end{align}
As $\thetaopt$ is unique, $\hat \delta$ is well-defined, and strictly positive by definition of $E$.
By \Cref{eq:optimality_subdiff}, the support of any solution is included in the equicorrelation set.
By \Cref{assum:uniqueness}, we even have equality. %

We will now show that the coefficients outside the equicorrelation eventually vanish.
The proof requires to study the primal iterates after each update (instead of after each epoch), hence we use the notation $\tilde\beta^{(s)}$ for the primal iterate after the $s$-th update of coordinate descent.
This update only modifies the $j$-th coordinate, with ${s \equiv j - 1  \mod p}~$:
\begin{align}
  \tilde\beta^{(s + 1)}_j = \ST\left(h_j(\tilde\beta^{(s)}), \tfrac{\lambda}{l_j}\right) \enspace.
\end{align}
Note that at optimality, for every $j \in [p]$, one has: %
\begin{align}\label{eq:ST_optimality}
  \betaopt_j = \ST\left(h_j(\betaopt), \tfrac{\lambda}{l_j} \right) \enspace.
\end{align}
Let us consider an update $s \in \bbN$ of coordinate descent such that the updated coordinate $j$ verifies $\tilde\beta_j^{(s+1)}\neq 0$ and $j \notin E$, hence, $\betaopt_j = 0$.
Then:
\begin{align}\label{ineq:partial_identification}
        |\tilde\beta_j^{(s+1)}-\betaopt_j	|
    &=
        \left|\ST\left(h_j(\tilde\beta^{(s)}), \tfrac{\lambda}{l_j}\right)
              - \ST\left(h_j(\betaopt), \tfrac{\lambda}{l_j}\right) \right| \nonumber\\
    &\leq
        \left| h_j(\tilde\beta^{(s)}) - h_j(\betaopt) \right|
        - \left(\tfrac{\lambda}{l_j}
              -  \absin{ h_j(\betaopt) } \right)  \enspace,
\end{align}
where we used the following inequality \citep[Lemma 3.2]{Hale_Yin_Zhang2008}:
\begin{equation}
  \ST(x, \nu) \neq 0 , \ST(y,\nu) = 0 \implies
      |\ST(x,\nu) - \ST(y,\nu)|\leq|x-y| - (\nu -|y|) \enspace .
\end{equation}
Now notice that by definition of the saturation gap~\eqref{eq:def_saturation_gap}, and since $j \notin E~$:
\begin{align}
    & \frac{\lambda}{l_j} \left(1-\frac{|x_j^\top \nabla F(X\betaopt)|}{\lambda}\right) \geq \hat{\delta} \enspace \nonumber ,\\
    \mathrm{that \, is,} \quad
    & \frac{\lambda}{l_j} -  \absin{h_j(\betaopt)}  \geq \hat{\delta} \quad \text{ (using } \betaopt_j=0) \enspace . \label{ineq:partial_v2}
\end{align}
Combining \Cref{ineq:partial_identification,ineq:partial_v2} yields
\begin{align}\label{ineq:identification_gap_enters}
    |\tilde\beta_j^{(s+1)}-\betaopt_j	|
    &\leq
          \left| h_j(\tilde\beta^{(s)}) - h_j(\betaopt) \right| - \hat{\delta} \enspace.
\end{align}
This can only be true for a finite number of updates, since otherwise taking the limit would give $0 \leq - \hat \delta$, and $\hat \delta >0$ (Eq.~\eqref{eq:def_saturation_gap}).
Therefore, after a finite number of updates, $\tilde\beta_j^{(s)} = 0$ for $j \notin E$.

For $j \in E$, $\betaopt_j \neq 0$ by \Cref{assum:uniqueness}, so $\beta^{(t)}_j$ has the same sign eventually since it converges to $\betaopt_j$.

The proof for proximal gradient descent is a result  of \citet[Theorem 4.5]{Hale_Yin_Zhang2008}, who provide the bound $T \leq \normin{\tilde\beta^{(s)} - \betaopt}_2^2/\hat{\delta}^2$.
 \end{proof}
\section{A seminal example: the Lasso case}
\label{sec:lasso_case}

{\fontsize{4}{4}\selectfont
\begin{algorithm}[t]
\SetAlgoLined
\SetKwInOut{Input}{input}
\SetKwInOut{Init}{init}
\SetKwInOut{Parameter}{param}
\caption{\textsc{PG/cyclic CD for \Cref{eq:sparse_glm} with dual extrapolation}}
\label{alg:CD_ISTA_dual_extrapolation}
\Input{$X =[x_1 |\dots | x_p], y, \lambda, \beta^{(0)}, \epsilon$}
\Parameter{$T, K=5, \freq=10$}
\Init{$X\beta = X\beta^{(0)}, \theta^{(0)} = - \nabla F(X\beta^{(0)})/\max(\lambda, \normin{X^\top \nabla F(X\beta^{(0)})}_\infty)$}%

\For{$t = 1, \ldots, T$}
    {
        \If(\tcp*[h]{compute $\theta$ and gap every $f$ epoch only}){$t = 0 \mod \freq$
            }
            {$t' = t/ \freq $ \tcp*[l]{dual point indexing}

            $r^{(t')} = X\beta$

                compute $\thetaresiduals^{(t')}$ and $\thetaccel^{(t')}$ with \cref{eq:theta_res,eq:R_accel,eq:theta_accel}

                    {

                    $\theta^{(t')} =  \argmax \condset{\cD(\theta)}{\theta \in \{\theta^{(t'-1)}, \thetaccel^{(t')}, \thetaresiduals^{(t')}\}}$ \tcp*[l]{robust dual extr. with \eqref{eq:robust_extrapolation}}
                    }

                \lIf{$\cP(\beta^{(t)}) - \cD(\theta^{(t')})< \epsilon$}
                    {break}
            }
        \If(\tcp*[h]{proximal gradient descent:}){\upshape PG}{

            $X\beta = X \beta^{(t)}$

            $\beta^{(t + 1)} = \ST\Big(\beta^{(t)} - \frac{\gamma}{\normin{X^\top X}_2} X^\top \nabla F(X \beta), \frac{\lambda \gamma }{\normin{X^\top X}_2}\Big)$\label{algoline:update_pg}
            }
        \ElseIf(\tcp*[h]{cyclic coordinate descent:}){\upshape CD}{
            \For{$j = 1, \ldots, p$}
                {

                $\beta^{(t + 1)}_j = \ST\Big(\beta^{(t)}_j - \tfrac{\gamma x_j^\top \nabla F(X\beta)}{\normin{x_j}^2}), \tfrac{\gamma \lambda}{\normin{x_j}^2}\Big)$\label{algoline:update_cd}

                        $X\beta \pluseq ( \beta^{(t + 1)}_j - \beta^{(t)}_j) x_j$

                }
            }
    }
\Return{$\beta^{(t)}$, $\theta^{(t')}$}
\end{algorithm}
}

Dual extrapolation was originally proposed for the Lasso in the \celer algorithm \citep{Massias_Gramfort_Salmon18}.
As the VAR model holds exactly in this case, we first devote special attention to it.
We will make use of asymptotic VAR models and generalize \celer to all sparse GLMs in \Cref{sec:others_models}.

Using the identification property of coordinate descent and proximal gradient descent, we can formalize the VAR behavior of dual iterates.

\begin{proposition}\label{thm:var_CD_PG_lasso}
 When $(\beta^{(t)})_{t \in \bbN}$ is obtained by cyclic coordinate descent or proximal gradient descent applied to the Lasso problem, $(X\beta^{(t)})_{t \in \bbN}$ is a VAR sequence after sign identification.
\end{proposition}

\begin{proof}
Let us first recall that the strong convexity constant $\gamma$ is equal to 1 in the Lasso case.
Let $t \in \bbN$ denote an epoch after sign identification.
The respective updates of proximal gradient descent and coordinate descent are reminded in lines \ref{algoline:update_pg} and \ref{algoline:update_cd} of
\Cref{alg:CD_ISTA_dual_extrapolation}.
\noindent
\paragraph{\emph{Coordinate descent:}} Let $j_1, \dots, j_S$ be the indices of the support of $\betaopt$, in increasing order. As the sign is identified, coefficients outside the support are 0 and remain 0.
We decompose the $t$-th epoch of coordinate descent into individual coordinate updates: l\textsl{}et $\tilde \beta^{(0)} \in \bbR^p$ denote the initialization (\ie the beginning of the epoch, $\tilde \beta^{(0)} = \beta^{(t)}$), $\tilde \beta^{(1)} \in \bbR^p$ the iterate after coordinate $j_1$ has been updated, etc., up to $\tilde \beta^{(S)}$ after coordinate $j_S$ has been updated, \ie at the end of the epoch ($\tilde \beta^{(S)} = \beta^{(t+1)}$).

Let $s \in [S]$, then $\tilde \beta^{(s)}$ and $\tilde \beta^{(s-1)}$ are  equal everywhere, except at coordinate $j_s$:
  \begin{align}\label{eq:beta_js_CD}
    \tilde \beta^{(s)}_{j_s}
      &= \ST\left(\tilde \beta^{(s - 1)}_{j_s} + \frac{1}{\normin{x_{j_s}}^2} x_{j_s}^\top\left(y - X\tilde\beta^{(s - 1)}\right),
                                   \frac{\lambda}{\normin{x_{j_s}}^2}\right) %
         \nonumber\\
         &= \tilde \beta^{(s - 1)}_{j_s} + \frac{1}{\normin{x_{j_s}}^2} x_{j_s}^\top\left(y - X\tilde\beta^{(s - 1)}\right)
            - \frac{\lambda \sign(\betaopt_{j_s})}{\normin{x_{j_s}}^2} \enspace,
  \end{align}
  where we have used sign identification: $ \sign(\tilde\beta^{(s)}_{j_s}) =  \sign(\betaopt_{j_s})$. Therefore
 \begin{align}
    X\tilde \beta^{(s)} - X\tilde \beta^{(s-1)} &= x_{j_s}\left(\tilde \beta^{(s)}_{j_s} - \tilde \beta^{(s-1)}_{j_s} \right)
    \nonumber \\
    &=  x_{j_s} \left( \frac{x_{j_s}^\top (y - X\tilde \beta^{(s-1)}) - \lambda \sign(\betaopt_{j_s})}{\normin{x_{j_s}}^2}  \right)  %
          \nonumber \\
    &=  \frac{1}{\normin{x_{j_s}}^2}  x_{j_s} x_{j_s}^\top \left(y - X\tilde \beta^{(s-1)}\right) - \frac{\lambda \sign(\betaopt_{j_s})}{\normin{x_{j_s}}^2}  x_{j_s} \enspace.
  \end{align}
  This leads to the following linear recurrent equation:
  \begin{align}\label{eq:CD_is_var_proof}
    X\tilde \beta^{(s)}
     &=  \underbrace{\left(\Id_n - \frac{1}{\normin{x_{j_s}}^2}  x_{j_s} x_{j_s}^\top\right)}_{A_{s} \in \bbR^{n \times n}}
                    X\tilde \beta^{(s-1)} +
                      \underbrace{\frac{x_{j_s}^\top y - \lambda \sign(\betaopt_{j_s})}{\normin{x_{j_s}}^2}  x_{j_s}}_{b_{s} \in \bbR^n} \enspace.
  \end{align}
  Hence, one gets recursively
  \begin{align}
    X\tilde \beta^{(S)} &= A_S X\tilde\beta^{(S - 1)} + b_S %
    \nonumber \\
      &= A_{S} A_{S-1} X\tilde \beta^{(S-2)} + A_S b_{S-1} + b_S %
      \nonumber \\
      &= \underbrace{A_S \dots A_1 }_{A}
          X\tilde \beta^{(0)}
          + \underbrace{A_S \dots A_2 b_1 + \dots + A_S b_{S-1} + b_S}_{b} \enspace.
          \label{eq:VAR_cd_lasso_A1AS}
  \end{align}
  We can thus write the following VAR equations for $X\beta$ at the end of each epoch coordinate descent:
  \begin{align}
      X \beta^{(t+1)} &= A X\beta^{(t)} + b \enspace,
      \label{eq:VAR_cd_lasso} \\
      X \beta^{(t+1)} - X\betaopt &= A (X\beta^{(t)} - X  \betaopt)\enspace. %
  \end{align}
  \paragraph{\emph{Proximal gradient:}}  Let $\beta^{(t)}_\cS$, $\betaopt_\cS$ and $X_\cS$ denote respectively $\beta^{(t)}$, $\betaopt$ and $X$ restricted to features in the support $\cS(\betaopt)$.
  Notice that since we are in the identified sign regime, $X\beta^{(t)} = X_\cS \beta_\cS^{(t)}~$.
  With $L = \normin{X^\top X}_2~$, a proximal gradient descent update reads:
  \begin{align}
      \beta_\cS^{(t+ 1)}
        &= \ST\left(\beta_\cS^{(t)} - \tfrac{1}{L} X_\cS^\top(X_\cS \beta_\cS^{(t)} -y), \tfrac{\lambda}{L}\right) %
        \nonumber \\
        &=\beta_\cS^{(t)} - \tfrac{1}{L} X_\cS^\top \left(X_\cS \beta_\cS^{(t)} -y \right) -  \tfrac{\lambda}{L} \sign (\betaopt_\cS) %
            \nonumber \\
        &= \left(\Id_S  - \tfrac{1}{L} X_\cS^\top X_\cS\right) \beta_\cS^{(t)} + \tfrac{1}{L} X_\cS^\top y  - \tfrac{\lambda}{L} \sign (\betaopt_\cS) \enspace.
  \end{align}
  Hence the equivalent of \Cref{eq:VAR_cd_lasso} for proximal gradient descent is:
  \begin{equation}\label{eq:VAR_cd_ISTA}
    X \beta^{(t + 1)} = \left(\Id_n  - \tfrac{1}{L} X_\cS X_\cS^\top\right) X \beta^{(t)}
                      + \tfrac{1}{L} X_\cS X_\cS^\top y
                      - \tfrac{\lambda}{L} X_\cS \sign (\betaopt_\cS) \enspace.
  \end{equation}
\end{proof}

\begin{figure}
  \centering
  \includegraphics[width=7cm]{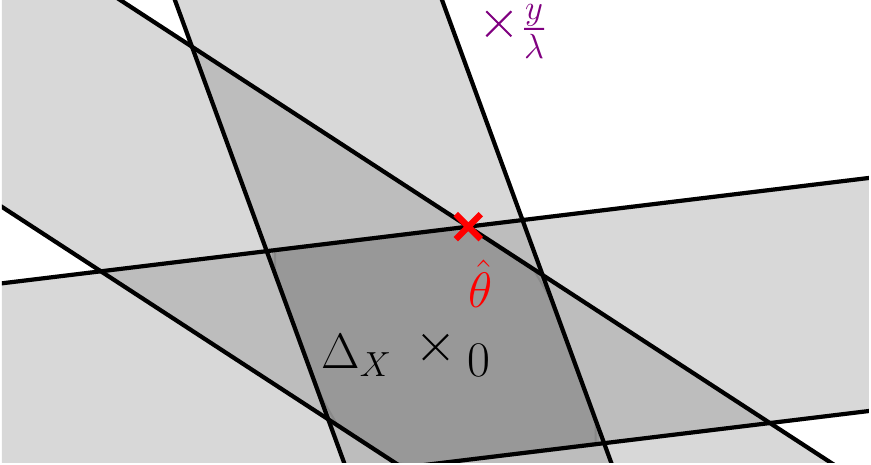}
  \includegraphics[width=7cm]{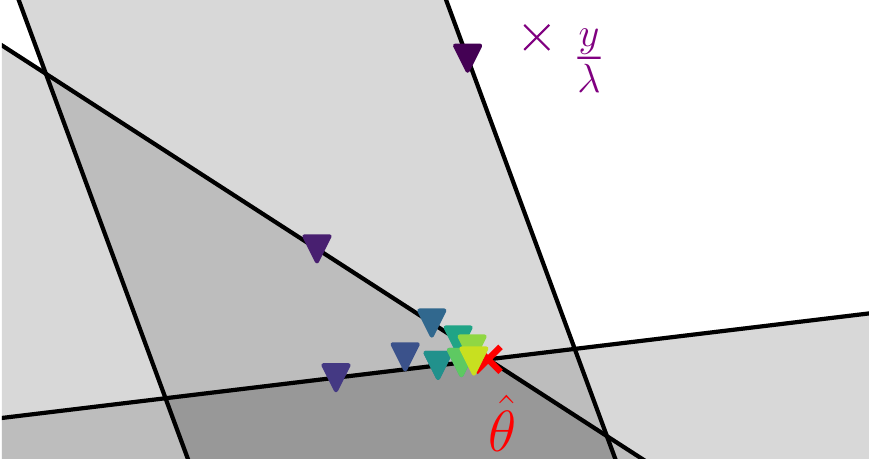}
  \caption{Illustration of the VAR nature of the dual iterates of the Lasso, on a toy dataset with $n=2$ and $p=3$. Left: dual of the Lasso problem; the dual optimum $\thetaopt$ is the projection of $y / \lambda$ onto $\Delta_X$. Right: sequence of residuals after each update of coordinate descent (first iterates in blue, last in yellow). After four updates, the iterates alternate geometrically between the same two constraint hyperplanes.}
  \label{fig:illustration_dual}
\end{figure}

\Cref{fig:illustration_dual} represents the Lasso dual for a toy problem and illustrates the VAR nature of $r^{(t)} / \lambda$.
As highlighted in \citet{Tibshirani17}, the iterates $r^{(t)} / \lambda$ correspond to the iterates of Dykstra's algorithm to project $y / \lambda$ onto $\Delta_X$.
During the first updates, the dual iterates do not have a regular trajectory.
However, after a certain number of updates (corresponding to sign identification), they alternate in a geometric fashion between two hyperplanes.
In this regime, it becomes beneficial to use extrapolation to obtain a point closer to $\thetaopt$.
\begin{remark}\label{rem:cyclic_vs_random_cd}
  \Cref{eq:VAR_cd_lasso_A1AS} shows why we combine extrapolation with \emph{cyclic} coordinate descent: if the coefficients are not always updated in the same order (see \citealt[Figure 1(c-d)]{Massias_Gramfort_Salmon18}), the matrix $A$ depends on the epoch, and the VAR structure may no longer hold.
\end{remark}
Having highlighted the VAR behavior of $(X\beta^{(t)})_{t \in \bbN}$, we can introduce our proposed dual extrapolation.
\begin{definition}[Extrapolated dual point for the Lasso]\label{def:extrapolated_dual_point_Lasso}
  For a fixed number $K$ of proximal gradient descent or coordinate descent epochs, let $r^{(t)}$ denote the residuals $y - X\beta^{(t)}$ at epoch $t$ of the algorithm. We define the extrapolated residuals
  \begin{equation}\label{eq:R_accel}
       r_{\mathrm{acc}}^{(t)} = \begin{cases}
                          r^{(t)}, &\mbox{if } t \leq K \enspace,\\
                          \displaystyle\sum_{k=1}^K c_k r^{(t + 1 - k)},    &\mbox{if } t > K \enspace.
                      \end{cases}
  \end{equation}
  where $c=(c_1,\dots,c_K)^\top \in \bbR^K$ is defined as in~\eqref{eq:closed_form_ck} with $U^{(t)} = [r^{(t + 1 - K)} - r^{(t  - K)}  , \ldots, r^{(t)} - r^{(t - 1)}] \in \bbR^{n \times K}$.
  Then, we define the extrapolated dual point as:
  \begin{equation}\label{eq:theta_accel}
      \thetaccel^{(t)} \eqdef r_{\mathrm{acc}}^{(t)} / \max(\lambda, \normin{X^\top r_{\mathrm{acc}}^{(t)}}_\infty)\enspace.
  \end{equation}
\end{definition}
In practice, we use $K = 5$ and do not compute $\thetaccel^{(t)}$ if $U^{(t)}{}^\top U^{(t)}$ cannot be inverted.
Additionally, to impose monotonicity of the dual objective, and guarantee a behavior at least as good at $\thetaresiduals$, we use as dual point at iteration $t$:
\begin{equation}\label{eq:robust_extrapolation}
  \theta^{(t)} = \displaystyle\argmax_{\theta \in \{\theta^{(t-1)}, \thetaccel^{(t)}, \thetaresiduals^{(t)}\}}\cD(\theta)\enspace.
\end{equation}

There are two reasons why the results of \Cref{prop:extrapolation_var} cannot be straightforwardly applied to \Cref{eq:theta_accel}:
  \begin{enumerate}
    \item the analysis by \citet{Scieur_Bach_Daspremont16} requires $A$ to be symmetrical, which is the case for proximal gradient descent but not for cyclic coordinate descent
    (as $\Id_n - x_{j_s} x_{j_s}^\top / \normin{x_{j_s}}^2$ and $\Id_n - x_{j_{s'}} x_{j_{s'}}^\top / \normin{x_{j_{s'}}}^2$ only commute if $x_{j_s}$ and $x_{j_{s'}}$ are collinear).
    To circumvent this issue, we can make $A$ symmetrical:
    instead of considering cyclic updates, we could consider that iterates $\beta^{(t)}$ are produced by a cyclic pass over the coordinates, \emph{followed by a cyclic pass over the coordinates in reverse order}.
    The matrix of the VAR in this case is no longer $A=A_S \dots A_1$, but $A_1\dots A_S A_S \dots A_1=A_1^{\top}\dots A_S^{\top} A_S \dots A_1=A^\top A$ (the $A_s$'s are symmetrical).
    Experiments of \Cref{sec:experiments}, where a simple cyclic order is used, tend to indicate that there is in fact no need for $A$ to be symmetrical.

    \item for both proximal gradient and coordinate descent we have $\norm{A} = 1$ instead of $\norm{A} < 1$ as soon as $S < n$: if the support of $\betaopt$ is of size smaller than $n$ ($S < n$), 1 is an eigenvalue of $A$.
    Indeed, for coordinate descent, if $S < n$, there exists a vector $u \in \bbR^n$, orthogonal to the $S$ vectors $x_{j_1}, \dots, x_{j_S}$.
    The matrix $A_s = \Id_n - \frac{1}{\normin{x_{j_s}}^2}  x_{j_s} x_{j_s}^\top$ being the orthogonal projection onto $\Span(x_{j_s})^{\bot}$, we therefore have $A_s u = u$ for every $s \in [S]$, hence $Au = u$.
    For proximal gradient descent, $\tfrac{1}{L}\tilde{X_S} \tilde{X_S}^\top$ is not invertible when $S < n$, hence 1 is an eigenvalue of $\Id_n - \tfrac{1}{L} \tilde{X_S} \tilde{X_S}^\top$.
    This seems to contradict the convergence of the VAR sequence but is addressed in \Cref{prop:modulus_one,prop:eigenspace_orthogonal}.
  \end{enumerate}

\begin{lemma}\label{prop:modulus_one}
  For coordinate descent, if an eigenvalue of $A = A_S \dots A_1$ has modulus 1, it is equal to 1.
\end{lemma}

\begin{proof}
  The matrix $A_s = \Id_n - \frac{1}{\normin{x_{j_s}}^2}  x_{j_s} x_{j_s}^\top$ is the orthogonal projection onto $\Span(x_{j_s})^{\bot}$.
  Hence,
  \begin{equation}\label{eq:Aj_orthogonal_projection}
      \forall x \in \bbR^n, \norm{A_s x} = \norm{x}  \implies A_s x = x\enspace.
  \end{equation}
    Let $(\mu, x) \in \bbC \times \bbR^n$ \st $\abs{\mu} = 1$, $\norm{x} = 1$ and $A x = \mu x$.
    This means $\norm{Ax} = 1$.
    Because $\norm{A_1 x} < 1 \implies
    \norm{A_S  \dots A_1 x} \leq \norm{A_S  \dots A_2} \normin{A_1 x} < 1
 \implies \norm{Ax} < 1$, we must have $\norm{A_1 x} \geq 1$.
    Since it holds that $\norm{A_1 x} \leq \norm{x} = 1$, we have $\norm{A_1 x} = \norm{x}$, thus $A_1 x = x$ because $A_1$ is an orthogonal projection.
    By a similar reasoning, $A_2 x = x$, etc. up to $A_S x = x$, hence $Ax = x$ and $\mu = 1$.
\end{proof}

\begin{lemma}\label{prop:eigenspace_orthogonal}
  For coordinate descent (resp. proximal gradient descent) applied to solve the Lasso, the VAR parameters $A \in \bbR^{n \times n}$ and $b\in \bbR^n$ defined in \eqref{eq:VAR_cd_lasso_A1AS} (resp.~\eqref{eq:VAR_cd_ISTA})  satisfy $b \in \Ker(\Id_n - A)^\perp$.
\end{lemma}

\begin{proof}
  \paragraph{\emph{Coordinate descent case:}} Let us remind that $b=A_S \dots A_2 b_1 + \dots + A_S b_{S-1} + b_S$ in this case, with $b_s = x_{j_s}^\top y - \lambda \sign(\betaopt_{j_s})  x_{j_s}/\normin{x_{j_s}}^2$.
  Let $v \in \Ker(\Id_n - A)$. Following the proof of \Cref{prop:modulus_one}, we have $A_1 v = \dots = A_S v = v$.
  For $s \in [S]$, since $A_s$ is the projection on $\Span(x_{j_s})^\bot$, this means that $v$ is orthogonal to $x_{j_s}$.
  Additionally, $v^\top A_S \dots A_{s + 1} b_s = (A_{s + 1} \dots A_S v)^\top b_s = v^\top b_s = 0$ since $b_s$ is co-linear to $x_{j_s}$.
  Thus, $v$ is orthogonal to the $S$ terms which compose $b$, and $b \perp \Ker(\Id_n - A)$.

  \medskip
  \noindent
  \paragraph{\emph{Proximal gradient descent case:}}
  Let $v \in \Ker(\Id_n - A) = \Ker(X_\cS X_\cS^\top)$.
  We have $v^\top X_\cS X_\cS^\top v = 0 = \norm{X_\cS^\top v}^2$, hence $X_\cS^\top v = 0.$
  It is now clear that $v^\top b  = v^\top(- X_\cS X_\cS^\top y + \lambda X_\cS \sign \hat \beta)/L = 0$, hence
  $b \perp \Ker(\Id_n - A)$.
\end{proof}

\begin{proposition}
  \Cref{prop:extrapolation_var} holds for the residuals $r^{(t)}$ (produced either by proximal gradient descent or coordinate descent) even though $\normin{A}_2 =1$ in both cases.
\end{proposition}

\begin{proof}
Let us write $A = \bar{A} + \underline{A}$ with $\bar{A}$ the orthogonal projection on $\Ker(\Id_n - A)$.
By \Cref{prop:modulus_one}, $\normin{\underline{A}} < 1$.

Then, one can check that $A \underline{A}  = \underline{A}^2$ and $A \bar{A}  = \bar{A}^2 = \bar{A}$ and $Ab = \underline{A} b$.

Let $T$ be the epoch when support identification is achieved.
For $t \geq T$, we have
\begin{equation}
  r^{(t + 1)} = \underline{A} r^{(t)} + b + \bar{A}r^{(T)} \enspace .
\end{equation}
Indeed, it is trivially true for $t = T$ and if it holds for $t$,
\begin{align}
  r^{(t + 2)} &= A r^{(t + 1)} +   b \nonumber \\
  &= A (\underline{A} r^{(t)} + b + \bar{A}r^{(T)}) + b  \nonumber\\
  &= \underline{A}^2 r^{(t)} + \underline{A} b + \bar{A}r^{(T)} + b  \nonumber\\
  &= \underline{A}(\underline{A} r^{(t)} + b) + \bar{A}r^{(T)} + b \nonumber \\
  &= \underline{A} r^{(t + 1)} + \bar{A}r^{(T)} + b  \enspace.
\end{align}
Therefore, on the space $\Ker(\Id_n - A)$, the sequence $r^{(t)}$ is constant, and on its orthogonal $\Ker(\Id_n - A)^{\perp}$, it is a VAR sequence with associated matrix $\underline{A}$, whose spectral normal is strictly less than 1.
Therefore, the results of \Cref{prop:extrapolation_var} still hold.
\end{proof}

\begin{remark}[{Connection with primal-dual techniques}]

The goal of our construction is to improve convergence for the primal, by constructing a better dual certificate which provides a tighter stopping criterion.
In our scheme, the primal iterates directly influence the dual ones  -- either through the link equation (residuals rescaling), either through extrapolation -- but (apart from the influence of screening or working set selection), the primal iterates do not depend on the dual ones.
An alternative technique to improve convergence in the dual would be to solve simultaneously the primal and the dual.
The objective function in \Cref{eq:sparse_glm} is $F(X\beta) + \lambda \normin{\beta}_1$, hence since strong duality holds, an equivalent saddle point formulation is
\begin{problem}
    \begin{aligned}
        &\max_{\theta \in \bbR^n} \min_{\beta \in \bbR^p, z \in \bbR^n} F(z) + \lambda \normin{\beta}_1 + \lambda\theta^\top(z- X\beta) \enspace, \quad \text{\ie} \\
    &\max_{\theta \in \bbR^n} \min_{\beta \in \bbR^p} \underbrace{- F^*(- \lambda \theta) + \lambda \normin{\beta}_1 - \lambda\theta^\top X\beta}_{\cL(\beta, \theta)} \enspace.
    \end{aligned}
\end{problem}
To solve this problem, the primal-dual Arrow-Hurwicz \citep{Arrow_Hurwicz_Uzawa58} method alternates proximal maximization steps in $\theta$ and proximal minimization steps in $\beta$.
Here, the maximization step can even be performed exactly, yielding:
\begin{align}\label{eq:primal_dual}
\begin{cases}
    \beta^{(t + 1)} = \prox_{\lambda/L \normin{\cdot}_1} (\beta^{(t)} + \frac{\lambda}{L} X^\top \theta^{(t)}) \enspace, \\
    \lambda \nabla F^* (- \lambda \theta^{(t + 1)}) - \lambda X\beta^{(t + 1)} = 0 \enspace,
\end{cases}
\end{align}
and the last line is equivalent to $\theta^{(t + 1)} = - \nabla F (X\beta^{(t + 1)}) / \lambda$ \citep[Cor. 1.4.4]{Hiriart-Urruty_Lemarechal93b}, as in \Cref{eq:KKT}.
Using inertial variants of the scheme \eqref{eq:primal_dual}, such as the one by \citet{Chambolle_Pock11} is a potential lead, which we do not investigate further.
In our opinion, a more promising direction of research would be to design extrapolation methods for the primal-dual coordinate descent method of \citet{Fercoq_Bianchi15}, which is left to future work.
Finally, we are not aware of algorithms working directly in the dual; a reason for that is that getting feasible iterates by other means than rescaling requires the knowledge of the projection onto $\Delta_X$, which is as difficult as the primal (see \cite{Tibshirani17} on this matter).
\cite{Dunner_Forte_Takac_Jaggi16} use a so-called \emph{``Lipschitzing trick''} to make the dual unconstrained, but the rough bound $\lambda \normin{\hat \beta}_1 \leq F(\mathbf{0}_n)$ they used is likely to lead to poor values of convergence rate constants in practice.

\end{remark}

Although so far we have proven results for both coordinate descent and proximal gradient descent for the sake of generality, we observed that coordinate descent consistently converges faster.
Hence from now on, we only consider the latter.
\section{Generalized linear models}
\label{sec:others_models}
\subsection{Coordinate descent for $\ell_1$ regularization}
\label{sub:fixed_step_size}
\begin{proposition}[VAR for coordinate descent and Sparse GLM]
  \label{thm:var_cd_sparse_glm}
  When \Cref{eq:sparse_glm} is solved by cyclic coordinate descent, the dual iterates $(X\beta^{(t)})_{t \in \bbN}$ form an asymptotical VAR sequence.
\end{proposition}

\begin{proof}
  As in the proof of \Cref{thm:var_CD_PG_lasso}, we place ourselves in the identified sign regime, and consider only one epoch $t$ of CD: let $\tilde \beta^{(0)}$ denote the value of the primal iterate at the beginning of the epoch ($\tilde \beta^{(0)} = \beta^{(t)}$), and for $s \in [S]$, $\tilde \beta^{(s)} \in \bbR^p$ denotes its value after the $j_s$ coordinate has been updated ($\tilde \beta^{(S)} = \beta^{(t + 1)}$).
  Recall that in the framework of \Cref{eq:sparse_glm}, the data-fitting functions $f_i$ have $1/\gamma$-Lipschitz gradients, and ${\nabla F(u) = (f_1'(u_1), \dots, f_n'(u_n))}$.

For $s \in [S]$, $\tilde \beta^{(s)}$ and $\tilde \beta^{(s - 1)}$ are equal everywhere except at entry $j_s$, for which the coordinate descent update with fixed step size $\tfrac{\gamma}{\normin{x_{j_s}}^2}$ is
  \begin{align}\label{eq:cd_update_general}
    \tilde\beta^{(s)}_{j_s}
        &=  \ST\Big( \tilde\beta^{(s - 1)}_{j_s} - \tfrac{\gamma}{\normin{x_{j_s}}^2} x_{j_s}^\top \nabla F(X \tilde\beta^{(s - 1)}), \tfrac{\gamma}{\normin{x_{j_s}}^2} \lambda\Big) %
        \nonumber \\
        &= \tilde\beta^{(s - 1)}_{j_s} - \tfrac{\gamma}{\normin{x_{j_s}}^2} x_{j_s}^\top \nabla F(X \tilde\beta^{(s - 1)})
              - \tfrac{\gamma }{\normin{x_{j_s}}^2} \lambda \sign(\betaopt_{j_s}) \enspace.
  \end{align}
  Therefore,
    \begin{align}\label{eq:cd_general_fixed}
    X\tilde \beta^{(s)} - X\tilde \beta^{(s-1)}
      &=  x_{j_s} \left(\tilde\beta^{(s)}_{j_s} - \tilde\beta^{(s - 1)}_{j_s} \right) %
      \nonumber \\
     &=  x_{j_s}\left(- \tfrac{\gamma}{\normin{x_{j_s}}^2} x_{j_s}^\top \nabla F(X\tilde \beta^{(s - 1)} )
                - \tfrac{\gamma}{\normin{x_{j_s}}^2} \lambda \sign(\betaopt_{j_s})  \right)  \enspace.
    \end{align}
    Using point-wise linearization of the function $\nabla F$  around $X \betaopt$, we have:
  \begin{align}\label{eq:linearization_general}
    \nabla F(X \beta) &= \nabla F(X \betaopt) + D (X\beta - X\betaopt) + o(X\beta - X\betaopt) \enspace,
  \end{align}
  where $D \eqdef \diag(f_1''(\betaopt^\top \mathbf{x}_1), \dots, f_n''(\betaopt^\top \mathbf{x}_n)) \in \bbR^{n \times n}$. Therefore
    \begin{align}
    X\tilde \beta^{(s)} &= %
              \left(\Id_n - \tfrac{\gamma}{\normin{x_{j_s}}^2} x_{j_s} x_{j_s}^\top D \right)
              X\tilde \beta^{(s-1)}
               \nonumber\\
              &\phantom{=} +
                    \tfrac{\gamma}{\normin{x_{j_s}}^2} \left(x_{j_s}^\top (D X\betaopt - \nabla F(X\betaopt)) - \lambda \sign( \betaopt_{j_s})
                    \right) x_{j_s}   + o(X \tilde \beta^{(s)} - X \betaopt)\enspace, \nonumber\\
    D^{1 / 2}  X\tilde \beta^{(s)} &=
    \underbrace{\left(\Id_n- \tfrac{\gamma}{\normin{x_{j_s}}^2} D ^{1 / 2} x_{j_s} x_{j_s}^\top D ^{1 / 2}\right)}_{A_s} D ^{1 / 2}  X\tilde \beta^{(s-1)} \nonumber
    \\ &\phantom{= \quad} +
          {\underbrace{ \tfrac{\gamma}{\normin{x_{j_s}}^2}  x_{j_s}^\top (D X\betaopt)
                     D ^{1 / 2} x_{j_s}}_{b_s}}
                + o(X \tilde \beta^{(s)}  - X \betaopt)\enspace ,
  \end{align}
  since the subdifferential inclusion \eqref{eq:subdiff_inclusion} gives
  $ - x_{j_s}^\top \nabla F(X\betaopt) - \lambda \sign (\betaopt_{j_s}) = 0$.
  Thus, the sequence $(D ^{1 / 2} X \beta^{(t)})_{t \in \bbN}$ is an asymptotical VAR sequence:
  \begin{equation}
    D^{1 / 2} X \beta^{(t + 1)} = A_S \dots A_1 D^{1 / 2} X \beta^{(t)} + b_S + \ldots + A_S \ldots A_2 b_1 + o(X\beta^{(t)} - X\betaopt) \enspace,
  \end{equation}
   and so is $(X \beta^{(t)})_{t \in \bbN}$:
  \begin{align}\label{eq:var_glm}
     X \beta^{(t + 1)} = \underbrace{D ^{-\tfrac{1}{2}} A_S \dots A_1 D ^{\tfrac{1}{2}}}_{A} X \beta^{(t)}
                      + \underbrace{D^{-\tfrac{1}{2}} (b_S + \ldots + A_S \ldots A_2 b_1)}_{b} + o(X\beta^{(t)} - X\betaopt)\enspace.
  \end{align}
\end{proof}
\begin{proposition}\label{prop:eigenspace_orthogonal_glm}
  As in \Cref{prop:modulus_one,prop:eigenspace_orthogonal}, for the VAR parameters $A$ and $b$ defined in~\Cref{eq:var_glm}, 1 is the only eigenvalue of $A$ whose modulus is 1 and $b \perp \Ker(\Id_n - A)$.
\end{proposition}
\begin{proof}
  First, notice that as in the Lasso case, we have $\Id_n \succeq A_s \succeq 0$.
  Indeed, because $f''_i$ takes values in $]0, 1/ \gamma[$, $D^{1/2}$ exists and  $\tfrac{1}{\sqrt{\gamma}}\Id_n \succeq D^{1/2} \succeq 0$.
  For any $u \in \bbR^n$,
  \begin{align}
    u^\top D^{1/2}   x_{j_s} x_{j_s}^\top D^{1/2}u &= (x_{j_s}^\top D^{1/2} u)^2 \geq 0 ,\\
     \mbox{and } \quad \quad \quad \quad  x_{j_s}^\top D^{1/2} u &\leq \normin{x_{j_s}} \normin{D^{1/2} u}  \nonumber  \\
      &\leq  \normin{x_{j_s}} \normin{D^{1/2}} \normin{u} \nonumber \\
      &\leq  \tfrac{1}{\sqrt{\gamma}} \normin{x_{j_s}} \normin{u} \enspace ,
  \end{align}
  thus $\tfrac{\normin{x_{j_s}}^2}{\gamma}  \Id_n \succeq D^{1/2}   x_{j_s} x_{j_s}^\top D^{1/2} \succeq 0$
  and $\Id_n \succeq A_s \succeq 0$.

  However, contrary to the Lasso case, because $\norm{D^{1/2}x_{j_s}} \neq \sqrt{\gamma}\norm{x_{j_s}}$, $A_s$ is not the orthogonal projection on $(\Span D^{1/2}x_{j_s})^\perp$.
  Nevertheless, we still have $A_s = A_s^\top$, $\normin{A_s} \leq 1$, and for $v \in \bbR^n$, $A_s v = v$ means that $v^\top D^{1/2} x_{j_s} = 0$, so the proof of \Cref{prop:modulus_one} can be applied to show that the only eigenvalue of $A_S \ldots A_1$ which has modulus 1 is 1.
  Then, observing that $A = D ^{- 1 / 2} A_S \dots A_1 D ^{1 / 2}$ has the same spectrum as $A_S \dots A_1$ concludes the first part of the proof.

  For the second result, let $v \in \Ker (\Id_n - A)$, \ie  $Av = v$, hence $A_S \ldots A_1  D^{1/2} v =  D^{1/2}Av = D^{1/2} v$.
  Therefore $D^{1/2} v$ is a fixed point of $A_S \ldots A_1$, and as in the Lasso case this means that for all $s \in [S]$, $A_s  D^{1/2} v =  D^{1/2} v$ and $ (D^{1/2} v)^\top D^{1/2} x_{j_s} = 0$.
  Now recall that
  \begin{align}
      b &= D^{-1/2} (b_S + \ldots + A_S \ldots A_2 b_1) \enspace, \\
      b_s &= \tfrac{\gamma}{\normin{x_{j_s}}^2} \left(x_{j_s}^\top (D X\betaopt - \nabla F(X\betaopt)) -         \lambda \sign (\betaopt_{j_s})
                    \right) D ^{1 / 2} x_{j_s} \nonumber \\
            &=  \tfrac{\gamma}{\normin{x_{j_s}}^2} ( x_{j_s}^\top D X\betaopt)
             D ^{1 / 2} x_{j_s} \enspace.
  \end{align}
  Additionally, $v^\top D^{-1/2}A_S \dots A_{s + 1} b_s  =  (A_{s + 1} \dots A_S D^{-1/2}v)^\top b_s = (D^{-1/2} v)^\top b_s = 0$.
  Hence $v$ is orthogonal to all the terms which compose $b$, hence $v^\top b = 0$.
\end{proof}
\Cref{thm:var_cd_sparse_glm} and \Cref{prop:eigenspace_orthogonal_glm} show that we can construct an extrapolated dual point for any sparse GLM, by extrapolating the sequence $(r^{(t)} = X\beta^{(t)})_{t \in \bbN}$ with the construction of \Cref{eq:R_accel}, and creating a feasible point with:
\begin{equation}
  \thetaccel^{(t)}\eqdef  -\nabla F (r_{\mathrm{acc}}^{(t)}) / \max(\lambda, \normin{X^\top \nabla F( r_{\mathrm{acc}}^{(t)} )}_\infty) \enspace.
\end{equation}

\subsection{Multitask Lasso}
\label{sub:multitask_lasso}

Let $q \in \bbN$ be a number of tasks, and consider an observation matrix $Y \in \bbR^{n \times q}$, whose $i$-th row is the target in $\bbR^q$ for the $i$-th sample. For $\Beta \in \bbR^{p \times q}$, let $\normin{\Beta}_{2,1} = \sum_1^p \norm{\Beta_j}$ (with $\Beta_j \in \bbR^{1 \times q}$ the $j$-th row of $\Beta$).

\begin{definition}\label{def:multitask_lasso}

  The multitask Lasso estimator is defined as the solution of:
  \begin{equation}\label{eq:multitask_lasso}
    \hat \Beta \in \argmin_{\Beta \in \bbR^{n\times q}} \frac{1}{2} \norm{Y - X \Beta}_F^2
         + \lambda  \norm{\Beta}_{2, 1} \enspace.
  \end{equation}
  \end{definition}

Let $j_1 < \dots < j_S$ denote the (row-wise) support of $\hat \Beta$, and let $t$ denote an iteration after support identification.
Note that the guarantees of support identification for multitask Lasso requires more assumptions than the case of the standard Lasso. In particular it requires a source condition which depends on the design matrix $X$. This was investigated for instance by~\citet{vaiter2017model} when considering a proximal gradient descent algorithm.

Let  $\tilde\Beta^{(0)} = \Beta^{(t)}$, and for $s \in [S]$, let $\tilde\Beta^{(s)}$ denote the primal iterate after coordinate $j_s$ has been updated.
Let $s \in [S]$, with $\tilde\Beta^{(s)}$ and $\tilde\Beta^{(s - 1)}$ being equal everywhere, except for their $j_s$ row for which one iteration of proximal block coordinate descent gives, with $\phi(\Beta) \eqdef \Beta_{j_s}  + \frac{1}{\normin{x_{j_s}}^2} x_{j_s}^\top (Y - X\Beta) \in \bbR^{1\times q}$:
\begin{align}\label{eq:bst_update}
    \tilde\Beta^{(s)}_{j_s} = \left(1 - \frac{\lambda / \norm{x_{j_s}}^2}{\normin{\phi(\tilde\Beta^{(s - 1)})}}\right) \phi(\tilde\Beta^{(s - 1)})
    \enspace .
\end{align}
From \Cref{eq:bst_update},
\begin{align}
      X \tilde\Beta^{(s)} - X\tilde\Beta^{(s - 1)}
          &= x_{j_s} (\tilde\Beta^{(s)}_{j_s} - \tilde\Beta^{(s - 1)}_{j_s}) \nonumber \\
          &= x_{j_s}
               \left(
                    \frac{1}{\normin{x_{j_s}}^2} x_{j_s}^\top (Y - X\tilde\Beta^{(s- 1)} )
                    -  \frac{\lambda / \norm{x_{j_s}}^2}{\normin{\phi(\tilde\Beta^{(s - 1)})}} \phi(\tilde\Beta^{(s - 1)})
              \right) \enspace.
\end{align}

Using
\begin{align}
         \frac{x + h}{\normin{x + h}} &= \frac{x}{\normin{x}} + \frac{1}{\normin{x}} \Big(\Id - \frac{x x^\top}{\normin{x}^2} \Big) h + o(\normin{h}) \label{eq:dl_x_by_normx},
\end{align}
and introducing $\psi \eqdef e_{j_s}^\top -   \frac{1}{\normin{x_{j_s}}^2} x_{j_s}^\top X \in \bbR^{1 \times p}$, so that $\phi(\Beta) = \phi(\hat \Beta) + \psi(\Beta - \hat \Beta)$, one has the following linearization:
\begin{align}\label{eq:dl_phi_normphi}
  \frac{\phi(\Beta)}{\normin{\phi(\Beta)}}
     &= \frac{\phi(\hat \Beta)}{\normin{\phi(\hat \Beta) }}
         + \frac{1}{\normin{\phi(\hat \Beta)}}
                  \psi(\Beta - \hat \Beta) \left(\Id_q - \frac{\phi(\hat \Beta)^\top \phi(\hat \Beta) }{\normin{\phi(\hat\Beta)}^2}   \right)
                  + o(\Beta - \hat \Beta) \enspace,
\end{align}

which does not allow to exhibit a VAR structure, as $\Beta$ should appear only on the right.
Despite this negative result, empirical results of \Cref{sec:experiments} show that dual extrapolation still provides a tighter dual point in the identified support regime.
\celer's generalization to multitask Lasso consists in using
$d_j^{(t)}= (1 - \normin{x_{j}^\top \Theta^{(t)}})/{\normin{x_{j}}}$
with the dual iterate $\Theta^{(t)} \in \bbR^{n \times q}$.
The inner solver is cyclic block coordinate descent (BCD), and the extrapolation coefficients are obtained by solving \Cref{eq:minimization_c}, which is an easy to solve matrix least-squares problem.
\begin{remark}
    As a concluding remark, we point that for the three models studied here, a VAR structure in the dual implies a VAR structure in the primal, provided $X_{\cS(\hat\beta)}$ has full column rank.
    Indeed, for any matrix $B$ such that $B X_{\cS(\hat\beta)} = \Id_{\normin{\hat\beta}_0}$, after support identification one has $\beta^{(t + 1)}_{\cS(\hat\beta)} = B A X_{\cS(\hat\beta)} \beta^{(t)}_{\cS(\hat\beta)} + Bb$.
    This paves the way for applying the techniques introduced here to extrapolation in the primal, which we leave to future work.
\end{remark}
\section{Working sets}
\label{sec:screening_working_sets}

Being able to construct a better dual point leads to a tighter gap and a smaller upper bound in \Cref{eq:gap_safe_rule}, hence to more features being discarded and a greater speed-up for Gap Safe screening rules.
As we detail in this section, it can easily be integrated in a efficient working set policy.

\subsection{Improved working sets policy}
\label{sub:working_sets}

Working set methods originated in the domains of linear and quadratic programming \citep{Thompson_Tonge_Zionts66,Palacios_Lasdon_Engquist82,Myers_Shih88}, where they are called active set methods.

In the context of this paper, a working set approach starts by solving \Cref{eq:sparse_glm} restricted to a small set of features $\cW^{(0)} \subset [p]$ (the working set), then defines iteratively new working sets $\cW^{(t)}$ and solves a sequence of growing problems \citep{Kowalski_Weiss_Gramfort_Anthoine11,Boisbunon_Flamary_Rakoto14,DeSantis_Lucidi_Rinaldi16}.
It is easy to see that when $\cW^{(t)} \subsetneq \cW^{(t + 1)}$ and when the subproblems are solved up to the precision required for the whole problem, then working sets techniques converge.

It was shown by \citet{Massias_Gramfort_Salmon17} that every screening rule which writes
\begin{equation}\label{eq:dj}
    \forall j \in [p], \quad d_j > \tau \Rightarrow \betaopt_j = 0 \enspace,
\end{equation}
allows to define a working set policy.
For example for Gap Safe rules,
\begin{equation}\label{eq:djtheta}
    d_j = d_j(\theta) \eqdef  \frac{ 1 -  \absin{x_{j}^\top \theta}}{\norm{x_{j}}}\enspace,
\end{equation}
is defined as a function of a dual point $\theta \in \Delta_X$.
The value $d_j$ can be seen as measuring the importance of feature $j$, and so given an initial size $p^{(1)}$ the first working set can be defined as:
\begin{equation}\label{eq:first_working_set}
  \cW^{(1)} = \{j_1^{(1)},\dots, j_{p^{(1)}}^{(1)}\}\enspace,
\end{equation}
with $d_{j_1^{(1)}}(\theta) \leq \dots \leq d_{j^{(1)}_{p^{(1)}}}(\theta) < d_j(\theta), \, \forall j \notin \cW^{(0)}$, \ie the indices of the $p^{(1)}$ smallest values of $d(\theta)$.
Then, the \emph{subproblem solver} is launched on $X_{\cW^{(1)}}$.
New primal and dual iterates are returned, which allow to recompute $d_j$'s and define iteratively:
\begin{equation}\label{eq:working_set_iterative}
  \cW^{(t + 1)} = \{j_1^{(t + 1)},\dots, j_{p^{(t + 1)}}^{(t + 1)}\}\enspace,
\end{equation}
where we impose $d_j(\theta) = -1$ when $\beta_j^{(t)} \neq 0$ to keep the active features in the next working set.
As in \citet{Massias_Gramfort_Salmon18}, we choose $p^{(t)} = \min(p, 2 \normin{\beta^{(t)}}_0)$ to ensure a fast initial growth of the working set, and avoid growing too much when the support is nearly identified.
The stopping criterion for the inner solver on $\cW^{(t)}$ is to reach a gap lower than a fraction $\rho = 0.3$ of the duality gap for the whole problem, $\cP(\beta^{(t)}) - \cD(\theta^{(t)})$.
These adaptive working set policies are commonly used in practice \citep{Johnson_Guestrin15,Johnson_Guestrin18}.

Combined with coordinate descent as an inner solver, this algorithm was coined \celer (Constraint Elimination for the Lasso with Extrapolated Residuals) when addressing the Lasso problem.
The results of \Cref{sec:others_models} justify the use of dual extrapolation for any sparse GLM, thus enabling us to generalize \celer to the whole class of models (\Cref{alg:celer}).

{\fontsize{4}{4}\selectfont
\begin{algorithm}[t]
\setcounter{AlgoLine}{0}
\SetKwInOut{Input}{input}
\SetKwInOut{Init}{init}
\SetKwInOut{Parameter}{param}
\caption{\celer for \Cref{eq:sparse_glm}}
\label{alg:celer}
\Input{$X, y, \lambda, \beta^{(0)}, \theta^{(0)}$}
\Parameter{$K=5, p^{(1)}=100,\epsilon, \maxws$}
\Init{ $\cW^{(0)} = \emptyset$}%
\lIf(\tcp*[h]{warm start}){$\beta^{(0)} \neq \bold{0}_p$}{$p^{(1)} = |\cS(\beta^{(0)})|$}%
\For{
        $t = 1, \ldots, \emph{\maxws}$
    }
    {
        compute $\thetaresiduals^{(t)}$ \tcp*[l]{\Cref{eq:theta_res}}

        \If{\emph{solver is Prox-Celer}}
        {
        do $K$ passes of CD on the support of $\beta^{(t)}$, extrapolate to produce  $\thetaccel^{(t-1)}$\label{algoline:extra_for_ws}

        $\thetainner^{(t-1)} = \argmax_{\theta \in \{\theta^{(t-1)}, \thetainner^{(t - 1)}\}} \cD(\theta)$
        }

        $\theta^{(t)} = \argmax_{\theta \in \{\theta^{(t-1)}, \thetainner^{(t - 1)}, \thetaresiduals^{(t)}\}} \cD(\theta)$ %

        $g^{(t)} = \cP(\beta^{(t-1)}) - \cD(\theta^{(t)})$\tcp*[l]{global gap}
        \lIf{ $g^{(t)} \leq \epsilon$}{break}

        $\epsilon^{(t)}, \cW^{(t)} = \texttt{create\_WS()}$ \tcp{get tolerance and working set with \Cref{alg:create_WS}}

        \tcp{Subproblem solver is \Cref{alg:CD_ISTA_dual_extrapolation} or \labelcref{alg:newton_celer} for Prox-Celer:}

        get  \!$\tilde{\beta}^{(t)}, \thetainner^{(t)}$ with subproblem solver applied to $(X_{\cW^{(t)}}, y, \lambda, (\beta^{(t - 1)})_{\cW^{(t)}}, \epsilon^{(t)})$

        $\thetainner^{(t)} = \thetainner^{(t)} / \max(1, \normin{X^\top \thetainner^{(t)}}_\infty)$

        set $\beta^{(t)} = \bold{0}_p$ and $(\beta^{(t)})_{\cW^{(t)}} = \tilde{\beta}^{(t)}$

    }
\Return{$\beta^{(t)}, \theta^{(t)}$}
\end{algorithm}
}
{\fontsize{4}{4}\selectfont
\begin{algorithm}[t]
\setcounter{AlgoLine}{0}
\SetKwInOut{Input}{input}
\SetKwInOut{Init}{init}
\SetKwInOut{Parameter}{param}
\caption{$\texttt{create\_WS}$}
\label{alg:create_WS}
\Input{$X, y, \lambda, \beta^{(t- 1)}, \theta^{(t)}, \cW^{(t-1)}, g^{(t)} $}
\Parameter{$p^{(1)}=100, \rho=0.3$}
\Init{$d = \mathbf{0}_p$}

\For{$j = 1, \ldots, p$}
{
    \lIf%
    { $\beta^{(t- 1)}_j \neq 0$}
    {$d_j^{(t)}= - 1$}
    \lElse{$d_j^{(t)} = (1 - \absin{x_{j}^\top \theta^{(t)}})/{\normin{x_{j}}}$}
}

    $\epsilon^{(t)}= \rho g^{(t)}$

    \lIf{$t \geq 2$}{$p^{(t)}= \min(2  \normin{\beta^{(t-1)}}_0, p)$}

$\cW^{(t)} = \condsetin{j\in [p]}{ d^{(t)}_j \text{ among } p^{(t)} \text{ smallest values of} \; d^{(t)}}$

\Return{$\epsilon^{(t)}, \cW^{(t)}$}
\end{algorithm}
}

\subsection{Newton-\celer}
\label{sub:newton_celer}
When using a squared $\ell_2$ loss, the curvature of the loss is constant: for the Lasso and multitask Lasso, the Hessian does not depend on the current iterate.
This is however not true for other GLM data fitting terms, \eg Logistic regression, for which taking into account the second order information proves to be very useful for fast convergence
\citep{Hsieh_Sustik_Dhillon_Ravikumar14}.
To leverage this information, we can use a prox-Newton method \citep{Lee_Sun_Saunders12,Scheinberg_Tang13} as inner solver; an advantage of dual extrapolation is that it can be combined with \emph{any} inner solver, as we detail below.
For reproducibility and completeness, we first briefly detail the Prox-Newton procedure used.
In the following and in \Cref{alg:newton_celer,alg:newton_direction,alg:backtracking} we focus on a single subproblem optimization, so for lighter notation we assume that the design matrix $X$ is already restricted to features in the working set.
The reader should be aware that in the rest of this section, $\beta$, $X$ and $p$ in fact refers to $\beta_{\cW^{(t)}}$, $X_{\cW^{(t)}}$, and $p^{(t)}$.

Writing the data-fitting term $f (\beta) = F(X\beta)$, we have $\nabla^2 f(\beta) = X^\top D X$,
where $D \in \bbR^{n \times n}$ is diagonal with $f_i''(\beta^\top \bfx_i)$ as its $i$-th diagonal entry.
Using $H = \nabla^2 f(\beta^{(t)})$ we can approximate the primal objective by\footnote{$H$ and $D$ should read $H^{(t)}$ and $D^{(t)}$ as they depend on $\beta^{(t)}$; we omit the exponent for brevity.}
\begin{align}\label{eq:quadratic_approx}
    f(\beta^{(t)}) + \nabla f(\beta^{(t)})^\top (\beta - \beta^{(t)})
                + \frac{1}{2}(\beta - \beta^{(t)})^\top H (\beta - \beta^{(t)})
            + \lambda \norm{\beta}_1 \enspace.
\end{align}
Minimizing this approximation yields the direction $\Delta^{(t)}$ for the \emph{proximal Newton} step:
\begin{align}\label{eq:prox_grad}
    \Delta^{(t )} + \beta^{(t)}
        &= \argmin_\beta \frac{1}{2} \norm{\beta - \beta^{(t)} + H^{-1} \nabla f(\beta^{(t)})}^2_{H} + \lambda \norm{\beta}_1
        \enspace.
\end{align}
Then, a step size $\alpha^{(t)}$ is found by backtracking line search (\Cref{alg:backtracking}), and:
\begin{equation}\label{eq:prox_newton_update}
    \beta^{(t + 1)} = \beta^{(t)} + \alpha^{(t)} \Delta^{(t)} \enspace.
\end{equation}
The approximation in \eqref{eq:quadratic_approx} is the sum of a quadratic function and a $\ell_1$ penalty, hence it can be minimized with proximal coordinate descent.
Since $H = X^\top D X$, coordinate descent can be implemented efficiently by keeping the model fit $X \beta$ up-to-date.
The algorithm is summarized in \Cref{alg:newton_direction}.
{\fontsize{4}{4}\selectfont
\begin{algorithm}[t]
\setcounter{AlgoLine}{0}
\SetKwInOut{Input}{input}
\SetKwInOut{Init}{init}
\SetKwInOut{Parameter}{param}
\caption{\textsc{Prox-Newton subproblem solver} (illustrated on logistic regression)}
\label{alg:newton_celer}
\Input{$X =[x_1 |\dots | x_p] \in \bbR^{n\times p}, y \in \bbR^n, \lambda, \beta^{(0)} \in \bbR^p, \epsilon$}
\Parameter{$\maxcditer=20, \maxbacktrackiter=10, K=5$}
\Init{$\Delta\beta = \mathbf{0}_p, X\Delta\beta = \mathbf{0}_n, \theta^{(0)} = \mathbf{0}_n, D = \mathbf{0}_{n \times n}, L = \mathbf{0}_p, $}%
\For{$t = 1, \ldots, T$}
{

\lFor{$ i = 1, \ldots, n$}
    {
        $D_{ii} = f_i''(\beta^\top \bfx_i) \left(= {\exp(y_i \beta^\top \bfx_i)} / {\left(1 + \exp(y_i \beta^\top \bfx_i) \right)^2} \right)$ }
  \lFor{$ j = 1, \ldots, p$}
    {
        $L_j = \langle x_j, x_j \rangle_D  \left(=  \sum_{i=1}^n x_{ij}^2 {\exp(y_i \beta^\top \bfx_i)}/{\left(1 + \exp(y_i \beta^\top \bfx_i) \right)^2} \right)$
    }

\lIf{$t = 1$}{$\texttt{\maxcditer} = 1$}
\lElse{$\texttt{\maxcditer} = 20$}
$\Delta \beta = \texttt{newton\_direction}(X, y, \beta^{(t - 1)}, D, L=(L_1, \dots, L_p), \maxcditer)$

$\alpha^{(t)} = \texttt{backtracking}(\Delta \beta, X\Delta \beta, y, \lambda, \maxbacktrackiter)$

$\beta^{(t)} = \beta^{(t- 1)} + \alpha^{(t)} \times \Delta \beta$

$\thetaresiduals^{(t)} = -\nabla F(X\beta^{(t)}) / \lambda \left(= - y / (\lambda \mathbf{1}_n + \lambda \exp(y \odot X \beta^{(t)})) \right)$

$\thetaresiduals^{(t)} = \thetaresiduals^{(t)} / \max(1, \normin{X^\top \thetaresiduals^{(t)}}_\infty)$

$\theta^{(t)} = \argmax_{\theta \in \{\theta^{(t- 1)}, \thetaresiduals\}}\cD(\theta)$
\If{$\cP(\beta^{(t)}) - \cD(\theta^{(t)})  < \epsilon$}{break}

}
\Return{$\beta^{(t)}$, $\theta^{(t)}$}
\end{algorithm}
}
{\fontsize{4}{4}\selectfont
\begin{algorithm}[t]
\setcounter{AlgoLine}{0}
\SetKwInOut{Input}{input}
\SetKwInOut{Init}{init}
\SetKwInOut{Parameter}{param}
\caption{$\texttt{newton\_direction}$ (illustrated on logistic regression)}
\label{alg:newton_direction}
\Input{$X =[x_1 |\dots | x_p] \in \bbR^{n\times p}, y \in \bbR^n, \beta \in \bbR^p, D \in\bbR^{n \times n}, L = (x_j^\top D x_j)_j \in \bbR^p, \maxcditer$}
\Parameter{$\epsilon, \mincditer = 2$}
\Init{$\Delta\beta = \mathbf{0}_p, X\Delta\beta = \mathbf{0}_n$}%

\For{$k=1, \dots, \emph{\maxcditer}$}
    {
        $\tau = 0$ \tcp{stopping condition}
        \For{$j = 1 , \dots, p$}
        {
            $u_j = \beta_j + (\Delta\beta)_j$

            $\tilde u_j = \ST\left(\beta_j + (\Delta\beta)_j
            - \frac{1}{L_j} \left(x_j^\top \nabla F(X\beta^{(t)}) -  x_j^\top D X\Delta \beta  \right), \frac{\lambda}{L_j}\right)$

            $(\Delta\beta)_j = \tilde u_j - \beta_j$

            $X\Delta\beta \pluseq (\tilde u_j - u_j) x_j$

            $\tau \pluseq (\tilde u_j - u_j)^2 \times  L_j^2$

        }

        \lIf{$\tau \leq \epsilon$ { \rm and} $k \geq \mathrm{\emph{\mincditer}}$}{break}
    }
\Return{$\Delta \beta$}
\end{algorithm}
}

Contrary to coordinate descent, Newton steps do not lead to an asymptotic VAR, which is required to guarantee the success of dual extrapolation.
To address this issue, we compute $K$ passes of cyclic coordinate descent restricted to the support of the current estimate $\beta$ before defining a working set (\Cref{alg:celer}, line \ref{algoline:extra_for_ws}).
The $K$ values of $X\beta$ obtained allow for the computation of both $\thetaccel$ and $\thetaresiduals$.
The motivation for restricting the coordinate descent to the support of the current estimate $\beta$
comes from the observation that dual extrapolation proves particularly useful once the support is identified.

The Prox-Newton solver we use is detailed in \Cref{alg:newton_celer}.
When \Cref{alg:celer} is used with \Cref{alg:newton_celer} as inner solver, we refer to it as the Newton-\celer variant.

\paragraph{Values of parameters and implementation details}
In practice, Prox-Newton implementations such as GLMNET \citep{Friedman_Hastie_Tibshirani10}, newGLMNET \citep{Yuan_Ho_Lin12} or QUIC \citep{Hsieh_Sustik_Dhillon_Ravikumar14} only solve the direction approximately in \Cref{eq:prox_grad}.
How inexactly the problem is solved depends on some heuristic values.
For reproducibility, we expose the default values of these parameters as inputs to the algorithms.
Importantly, the variable $\maxcditer$ is set to 1 for the computation of the first Prox-Newton direction.
Experiments have indeed revealed that a rough Newton direction for the first update was sufficient and resulted in a substantial speed-up.
Other parameters are set based on existing Prox-Newton implementations such as \blitz.
{\fontsize{4}{4}\selectfont
\begin{algorithm}[t]
\setcounter{AlgoLine}{0}
\SetKwInOut{Input}{input}
\SetKwInOut{Init}{init}
\SetKwInOut{Parameter}{param}
\caption{$\texttt{backtracking}$ (illustrated on logistic regression)}
\label{alg:backtracking}
\Input{$\Delta \beta, X\Delta \beta, \lambda$}
\Parameter{$\maxbacktrackiter = 20$}
\Init{$\alpha = 1$}%
\For{$k = 1, \dots, \text{\upshape \maxbacktrackiter}$}
    {
      $\delta = 0$

        \For{$j = 1, \dots, p$}
            {

                \lIf{$\beta_j + \alpha \times (\Delta \beta)_j < 0$}{$\delta \minuseq \lambda  (\Delta \beta)_j$}
                \lElseIf{$\beta_j + \alpha \times (\Delta \beta)_j > 0$}{$\delta \pluseq \lambda  (\Delta \beta)_j$}
                \lElseIf{$\beta_j + \alpha \times (\Delta \beta)_j = 0$}{$\delta \minuseq \lambda \abs{(\Delta \beta)_j}$}
            }

        $\aux = \nabla F (X\beta + \alpha \times  X\Delta\beta)  \left(= - y \odot \sigma(-y \odot (X\beta + \alpha \times X\Delta\beta)) \right)$

        $\delta \pluseq (X\Delta \beta)^\top \aux$

        \lIf{$\delta < 0$}{break}
        \lElse{$\alpha = \alpha / 2$}
    }
\Return{$\alpha$}
\end{algorithm}
}
\section{Experiments}
\label{sec:experiments}
In this section, we numerically illustrate the benefits of dual extrapolation on various data sets.
Implementation is done in Python, Cython \citep{Behnel_etal11} and numba \citep{Lam_Pitrou_Seibert15} for the low-level critical parts.
The solvers exactly follow the \texttt{scikit-learn} API \citep{Pedregosa_etal11,sklearn_api}, so that \celer can be used as a drop-in replacement in existing code.
The package is available under BSD3 license at \url{https://github.com/mathurinm/celer},
 with documentation and examples at \url{https://mathurinm.github.io/celer}.

In all this section, the estimator-specific $\lambda_{\max}$ refers to the smallest value giving a null solution (for instance $\lambda_{\max}=\normin{X^\top y}_\infty$ in the Lasso case, $\lambda_{\max}=\normin{X^\top y}_\infty / 2$ for sparse logistic regression, and $\lambda_{\max}=\normin{X^\top Y}_{2, \infty}$ for the Multitask Lasso).

 \begin{table}[h!]
     \centering
     \caption{Characteristics of datasets used}
     \label{tab:characteristics_datasets}
     \begin{center}
     \begin{small}
     \begin{tabular}{lcccc}
         \toprule
         name   & $n$ & $p$ & $q$ & density \\
         \midrule
         \emph{leukemia}      & $72$    & $\num{7129}$  & - & $1$   \\ %
          \emph{news20}     & $\num{19996}$   & $\num{632983}$  & - & $6.1\, 10^{-4}$   \\ %
          \emph{rcv1\_train} & $\num{20242}$ & $\num{19960}$  & - & $3.7 \, 10^{-3}$ \\
          \emph{finance (log1p)} & $\num{16087}$   & $ \num{1668738}$ & -  & $3.4 \, 10^{-3}$ \\
          Magnetoencephalography (MEG) & $305$ & $\num{7498}$ & 49 & $1$ \\
         \bottomrule
     \end{tabular}
     \end{small}
     \end{center}
 \end{table}

\subsection{Illustration of dual extrapolation}

\begin{figure}[h!]
  \begin{subfigure}[b]{0.9\columnwidth}
    \centering
    \quad\quad
    \includegraphics[width=0.85\linewidth]{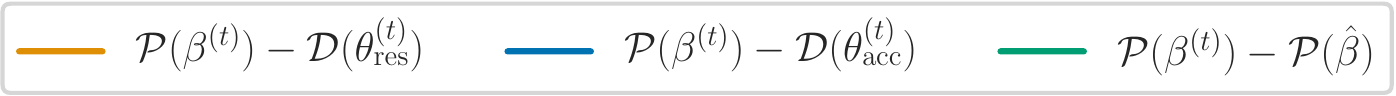} \\
  \end{subfigure}

  \begin{subfigure}[b]{0.5\columnwidth}
    \centering
    \includegraphics[width=0.95\linewidth]{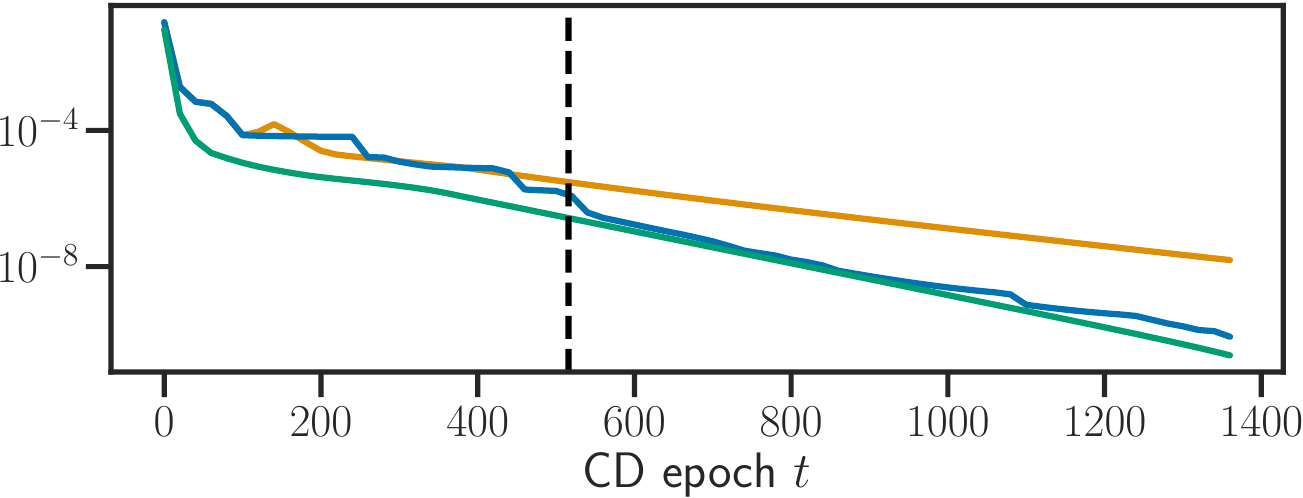}
    \caption{Lasso, on \emph{news20} for $\lambda = \lambda_{\max} / 5$.}
    \label{fig:VAR_lasso}
  \end{subfigure}
  \begin{subfigure}[b]{0.5\columnwidth}
    \centering
      \includegraphics[width=0.95\linewidth]{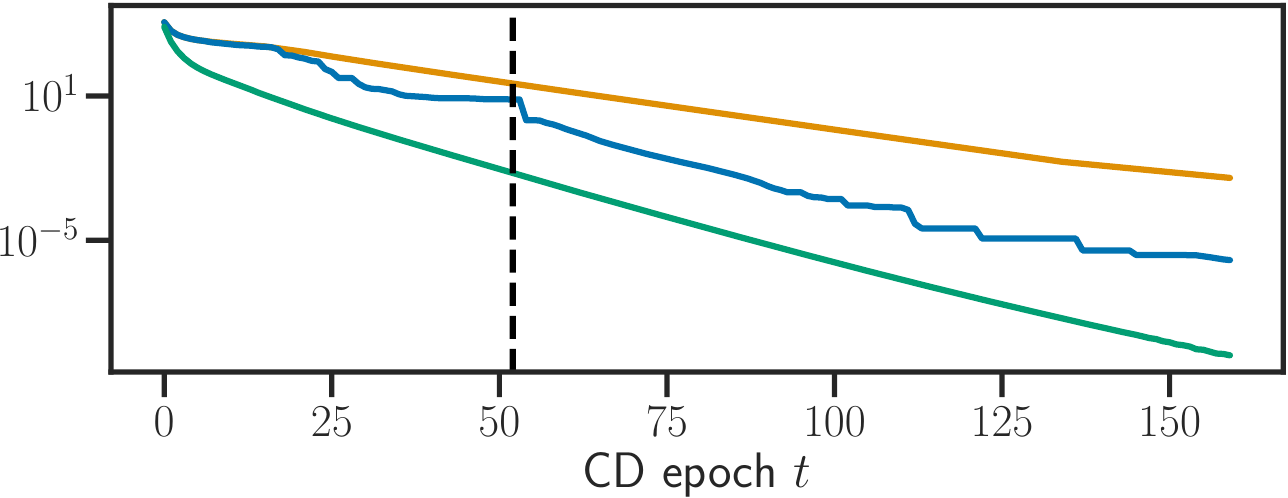}
    \caption{Log. reg., on \emph{rcv1 (train)} for $\lambda = \lambda_{\max} / 20$.}
    \label{fig:VAR_logreg}
  \end{subfigure}
  \begin{subfigure}[b]{\columnwidth}
      \centering
        \includegraphics[width=0.5\linewidth]{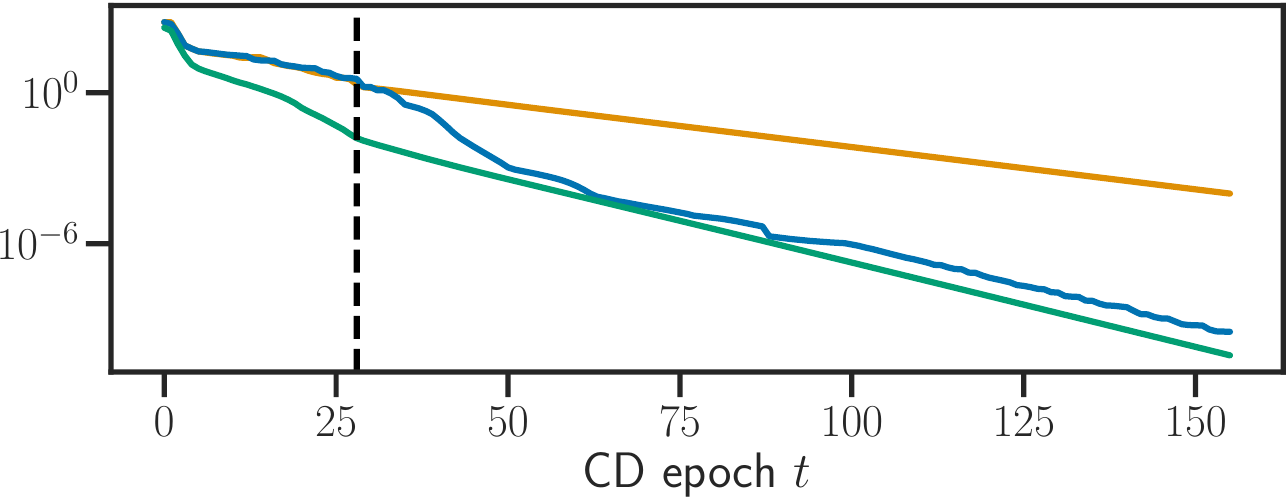}
      \caption{Multitask Lasso, on MEG data for $\lambda = \lambda_{\max} / 10$.}
    \label{fig:VAR_MTL}
  \end{subfigure}
  \caption{Dual objectives with classical and proposed approach, for Lasso (top left), Logistic regression (top right), Multitask Lasso (bottom).
           The dashed line marks sign identification (support identification for Multitask Lasso).}
\end{figure}

For the Lasso (\Cref{fig:VAR_lasso}), Logistic regression (\Cref{fig:VAR_logreg}) and Multitask Lasso (\Cref{fig:VAR_MTL}), we illustrate the applicability of dual extrapolation.
Monotonicity of the duality gap computed with extrapolation is enforced via the construction of \Cref{eq:robust_extrapolation}.
For all problems, the figures show that $\thetaccel$ gives a better dual objective after sign identification, with a duality gap sometimes even matching the suboptimality gap.
They also show that the behavior is stable before identification.

In particular, \Cref{fig:VAR_MTL} hints that dual extrapolation works in practice for the Multitask Lasso, even though there is no such result as sign identification, and we are not able to exhibit a VAR behavior for $(X\Beta^{(t)})_{t \in \bbN}$.
\Cref{fig:illustration_dual} suggests that the lower the stopping criterion threshold $\epsilon$, the higher the impact of dual extrapolation is.
However, when combined with screening, this improvement can be less visible in terms of time: if a large number of variables are screened before support identification, the later iterations concern a very small number of features.
In this case, decreasing the duality gap by running the solver longer after screening is not costly.

\subsection{Alternative exploitation of VAR structure}
Once one postulates that $\thetaopt$ is a linear combination of the $K$ most recent residuals, alternatives to our proposed dual extrapolation can be investigated to determine the coefficients of this combination.
This is particularly appealing in the Lasso case, for which the dual \Cref{pb:dual_sparse_glm} is:
\begin{problem}\label{pb:lasso_dual}
    \hat\theta = \argmax_{\theta \in \Delta_X} \; \frac12 \normin{y}^2 - \frac{\lambda^2}{2} \normin{y/\lambda - \theta}^2 \enspace.
\end{problem}
In this case, assuming that $\hat\theta$ belongs to $\Span(r^{(t)}, \ldots, r^{(t - K + 1)})$, we can reformulate \Cref{pb:lasso_dual} as a $K$-dimensional quadratic program, and directly optimize over the extrapolation coefficients.
If we write $R = [r^{(t)} | \ldots | r^{(t - K + 1)}] \in \bbR^{n \times K}$ and assume that $\hat\theta = R \hat c$, then \Cref{pb:lasso_dual} is equivalent to:
\begin{problem}\label{pb:lasso_dual_5d}
    \begin{aligned}
    \hat c &= \argmin_{c \in \bbR^K} \; \frac{1}{2} \normin{y/\lambda - Rc}^2  \text{\quad subject to \quad}  - \mathbf{1}_p \preceq X^\top R c \preceq \mathbf{1}_{p} \\
    &= \argmin_{c \in \bbR^K} \; \frac{\lambda}{2} c^\top (R^\top R) c - (R^\top y)^\top c  \text{\quad subject to \quad}  Ac \preceq \mathbf{1}_{2p} \enspace,
    \end{aligned}
\end{problem}
where $A^\top = [R^\top X ; - R^\top X]^\top \in \bbR^{2p \times K}$.
\Cref{pb:lasso_dual_5d} can be solved straightforwardly with solvers such as CVXPY \citep{cvxpy}, which we use in \Cref{fig:optim_5d}.
As visible on the latter, the QP approach seems to suffer more from numerical instabilities: at some iterations, CVXPY does not converge, which we represent by setting the dual objective to 0, hence the visible peaks.
Although it performs similarly to dual extrapolation at first, the QP dual point appears to eventually perform the same as residuals rescaling.
We do not perform an extensive time study of the compared approaches, but have observed that the $2p$ constraints of \Cref{pb:lasso_dual_5d} make it orders of magnitude slower to solve.
In practice, we therefore had to limit the experiment to the rather small \emph{leukemia} dataset to get reasonable running times.
Finally, the QP approach does not lead to simple optimization problems for sparse logistic regression and Multitask Lasso.

\begin{figure}
    \centering
    \includegraphics[width=\linewidth]{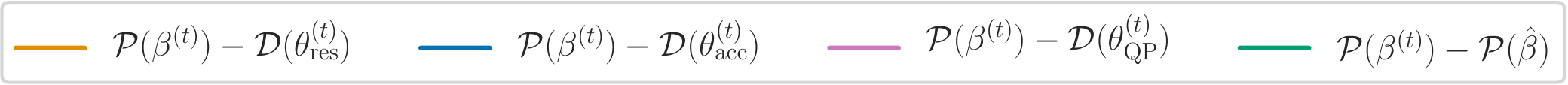} \\
    \includegraphics[width=0.9\linewidth]{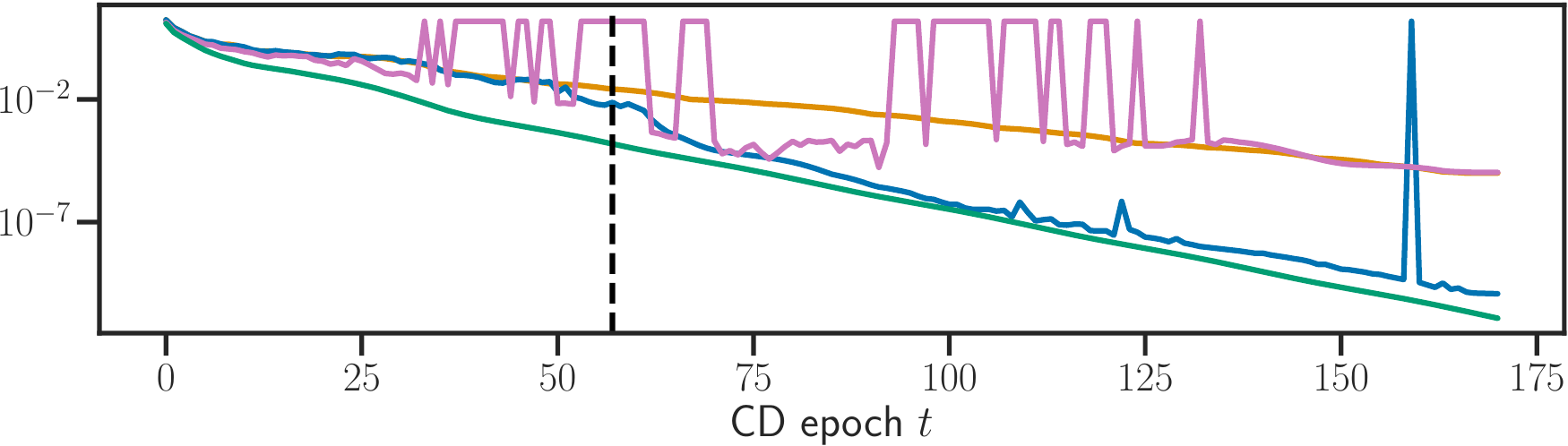}
    \caption{
    Duality gaps evaluated with rescaled residuals (yellow), our proposed dual extrapolation (blue), QP approach (purple) and optimal dual point (green), on the \emph{leukemia} dataset, with $\lambda = \lambda_{\max} / 5$ resulting in 23 non-zero coefficients at optimum. The dashed line marks support identification. Peaks occur because we set the duality gap to 0 when a method numerically fails to produce extrapolation coefficients.}
    \label{fig:optim_5d}
\end{figure}

\subsection{Improved screening and working set policy}
\label{sub:better_working_set_policy}

In order to have a stopping criterion scaling with $n$, the solvers are stopped when the duality gap goes below $\epsilon \times F(\mathbf{0}_n)$.
Features are normalized to have norm 1, and for sparse datasets, features with strictly less than 4 non-zero entries are removed.

\def\thisfigswidth{0.48\textwidth}

\subsubsection{Lasso}

\paragraph{Path computation}
For a fine (resp. coarse) grid of 100 (resp. 10) values of $\lambda$ geometrically distributed between $\lambda_{\max}$ and $\lambda_{\max}/100$,
the competing algorithms solve the Lasso on various real world datasets from LIBSVM\footnote{\url{https://www.csie.ntu.edu.tw/~cjlin/libsvmtools/datasets/}} \citep{Fan_Chang_Hsieh_Wang_Lin08}.
Warm start is used for all algorithms: except for the first value of $\lambda$, the algorithms are initialized with the solution obtained for the previous value of $\lambda$ on the path.
Note that paths are computed following a decreasing sequence of $\lambda$ (from high value to low).
Computing Lasso solutions for various values of $\lambda$ is a classical task, in cross-validation for example.
The values we choose for the grid are the default ones in \texttt{scikit-learn} or GLMNET.
For Gap Safe Rules (GSR), we use the strong warm start variant which was shown by \citet[Section 4.6.4]{Ndiaye_Fercoq_Gramfort_Salmon16b} to have the best performance.
We refer to ``GSR + extr.'' when, on top of this, our proposed dual extrapolation technique is used to create the dual points for screening.
To evaluate separately the performance of working sets and extrapolation, we also implement ``\celer w/o extr.'', \ie \Cref{alg:celer} without using extrapolated dual point.
Doing this, GSR can be compared to GSR + extrapolation, and \celer without extrapolation to \celer.
Finally, we also add the performance of Blitz\footnote{\url{https://github.com/tbjohns/BLitzL1}} \citep{Johnson_Guestrin18} and StingyCD\footnote{\url{https://github.com/tbjohns/StingyCD}} \citep{Johnson_Guestrin17}, the latter being a Lasso-specific coordinate descent designed to skip zero-to-zero updates.
Note that dual extrapolation could easily be combined with the update policy of StingyCD.
For fair comparison, all algorithms use the duality gap as a stopping criterion.

On \Cref{fig:ws_leukemia_lasso,fig:ws_news20_lasso,fig:ws_rcv1_lasso}, one can see that using acceleration systematically improves the performance of Gap Safe rules, up to a factor 3.
Similarly, dual extrapolation makes \celer more efficient than a WS approach without extrapolation (\blitz or \celer w/o extr.)
This improvement is more visible for low values of stopping criterion $\epsilon$, as dual extrapolation is beneficial once the support is identified.
Generally, working set approaches tend to perform better on coarse grid, while screening is beneficial on fine grids -- a finding corroborating Lasso experiments in \citet[Sec. 6.1]{Ndiaye_Fercoq_Gramfort_Salmon16b}.
Indeed, on a fine grid, the value of the regularizer $\lambda$ changes slowly and each solution on the grid is close to the previous one.
In this case, when warm-start is used, the initialization (approximate solution for the previous value of the regularizer) is close to the solution for the new value of the regularizer, and the duality gap always remains low, allowing to quickly screen features.
On the contrary, if the grid is coarse, each problem on the grid is quite different from the previous one.
Warm start here provides a less useful initialization as the duality gap is higher for the early iterations of each problem.
This results in a reduced efficiency of screening.

\begin{figure}[htbp]
    \centering
    \includegraphics[width=0.95\linewidth]{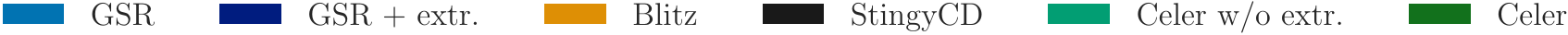} \\
    \includegraphics[width=0.45\linewidth]{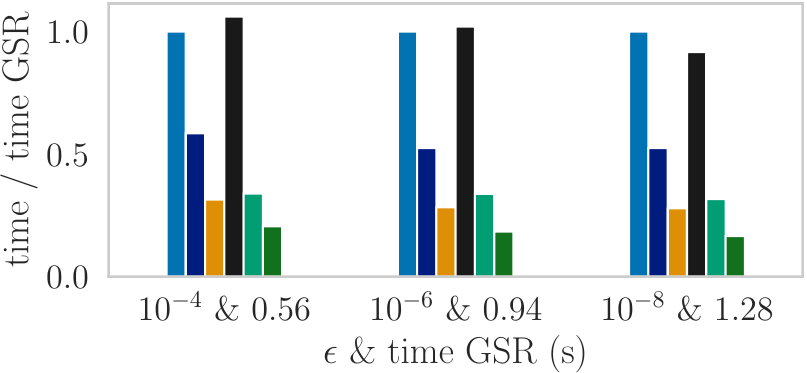}    \includegraphics[width=0.45\linewidth]{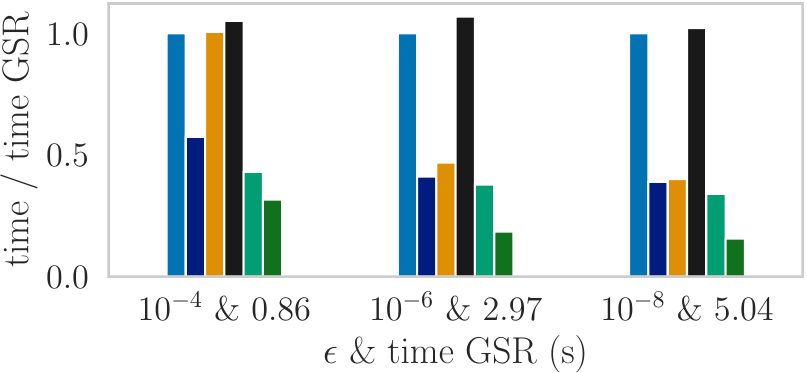}
    \caption{Time to compute a Lasso path from $\lambda_{\max}$ to $\lambda_{\max} / 100$ on the
    \emph{leukemia} dataset (\textbf{left}: coarse grid of 10 values, \textbf{right}: fine grid of 100 values).
    $\lambda_{\max} / 100$ gives a solution with \num{60} nonzero coefficients.}
    \label{fig:ws_leukemia_lasso}
\end{figure}

\begin{figure}[htbp]
    \centering
    \includegraphics[width=0.95\linewidth]{bench_blitz_gsr_lasso_wstingy_legend-crop} \\
    \includegraphics[width=0.45\linewidth]{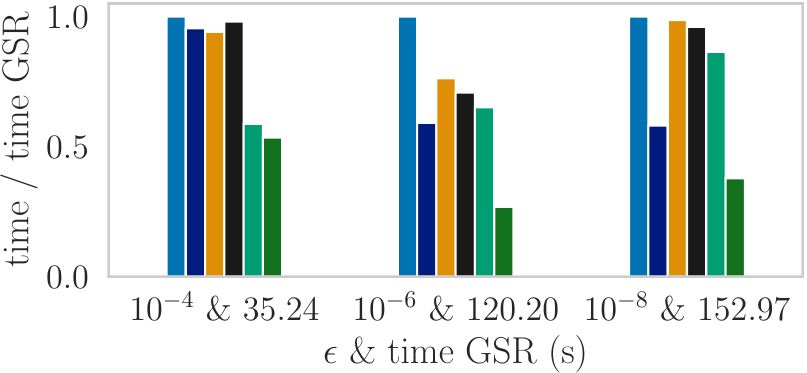}
    \includegraphics[width=0.45\linewidth]{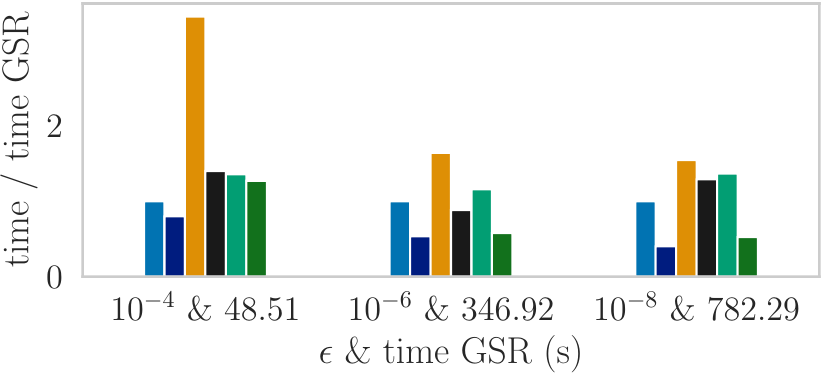}
    \caption{Time to compute a Lasso path from $\lambda_{\max}$ to $\lambda_{\max} / 100$ on the
    \emph{news20} dataset (\textbf{left}: coarse grid of 10 values, \textbf{right}: fine grid of 100 values).
    $\lambda_{\max} / 100$ gives a solution with \num{14817} nonzero coefficients.}
    \label{fig:ws_news20_lasso}
\end{figure}

\begin{figure}[htbp]
    \centering
    \includegraphics[width=0.95\linewidth]{bench_blitz_gsr_lasso_wstingy_legend-crop} \\
    \includegraphics[width=0.45\linewidth]{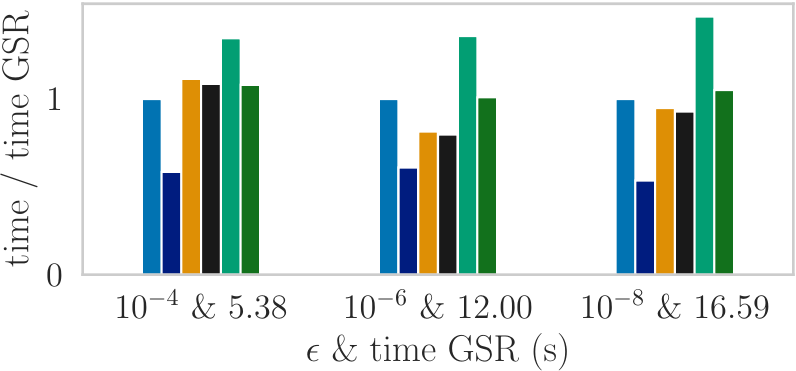}
    \includegraphics[width=0.45\linewidth]{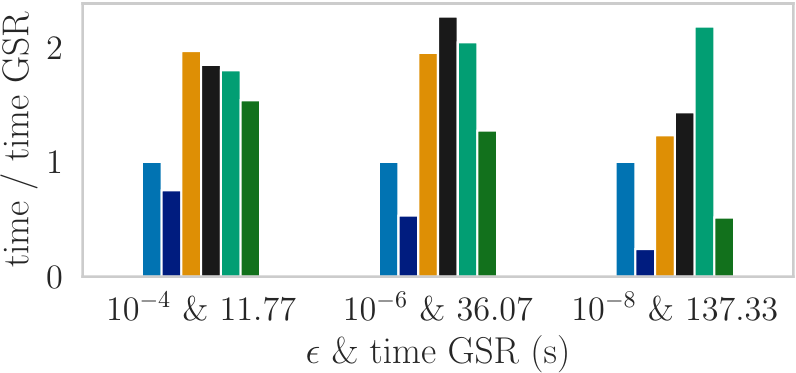}
    \caption{Time to compute a Lasso path from $\lambda_{\max}$ to $\lambda_{\max} / 100$ on the
    \emph{rcv1} dataset (\textbf{left}: coarse grid of 10 values, \textbf{right}: fine grid of 100 values).
    $\lambda_{\max} / 100$ gives a solution with \num{4610} nonzero coefficients.}
    \label{fig:ws_rcv1_lasso}
\end{figure}

\paragraph*{Single $\lambda$ }
\begin{table}[]
    \centering
    \caption{Computation time (in seconds) for \celer, \blitz and scikit-learn to reach a given precision $\epsilon$ on the \emph{Finance} dataset with $\lambda = \lambda_{\max}/20$ (without warm start: $\beta^{(0)}=\bold{0}_p$).}
    \label{tab:single_lambda_blitz_sklearn}
    \medskip
    \begin{center}
    \begin{small}
    \begin{tabular}{lcccc}
        \toprule
        $\epsilon$   & $10^{-2}$ & $10^{-3}$ & $10^{-4}$ & $10^{-6}$\\
        \midrule
        \celer       & $5$    & $7$    & $8$    &  $10$ \\
        \blitz        & $25$   & $26$   & $27$   & $30$ \\
        scikit-learn & $470$  & $\num{1350}$ & $\num{2390}$ &  - \\
        \bottomrule
    \end{tabular}
    \end{small}
    \end{center}
\end{table}
The performance observed in the previous paragraph is not only due to the sequential setting: in the experiment of \Cref{tab:single_lambda_blitz_sklearn}, we solve the Lasso for a single value of $\lambda=\lambda_{\max} / 20$.
The duality gap stopping criterion $\epsilon$ varies between $10^{-2}$ and $10^{-6}$.
\celer is orders of magnitude faster than scikit-learn, which uses vanilla CD.
The working set approach of \blitz is also outperformed, especially for low $\epsilon$ values.

\subsubsection{Logistic regression}

In this section, we evaluate separately the first order solvers (Gap Safe, Gap Safe with extrapolation, \celer with coordinate descent as inner solver), %
and the Prox-Newton solvers: \blitz, Newton-\celer with working set but without using dual extrapolation (PN WS), and Newton-\celer.

\begin{figure}[htbp]
    \centering
    \includegraphics[width=0.7\linewidth]{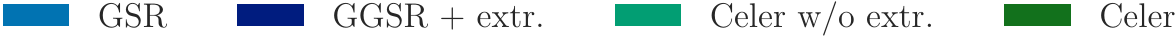} \\
    \includegraphics[width=\thisfigswidth]{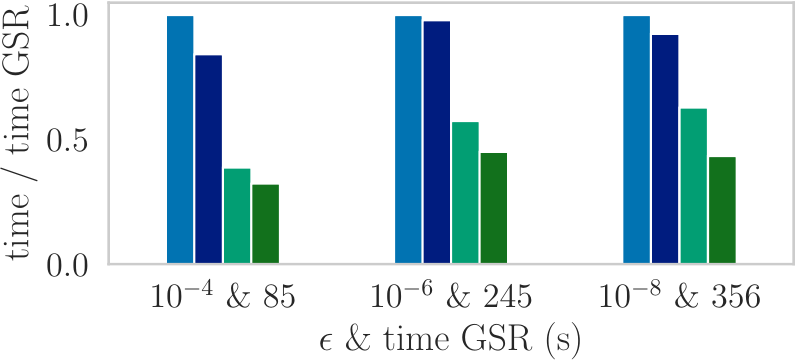}
    \includegraphics[width=\thisfigswidth]{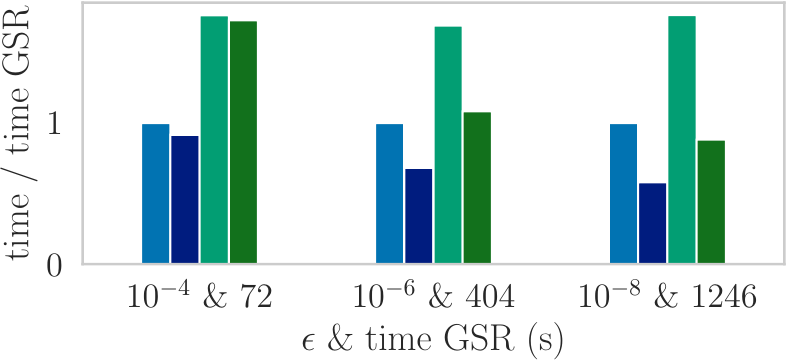}
    \caption{Time to compute a Logistic regression path from $\lambda_{\max}$ to $\lambda_{\max} / 100$ on the
    \emph{news20} dataset  (\textbf{left}: coarse grid of 10 values, \textbf{right}: fine grid of 100 values). $\lambda_{\max} / 100$ gives 5319 non-zero coefficients.}
    \label{fig:ws_rcv1_logreg}
\end{figure}

\Cref{fig:ws_rcv1_logreg} shows that when cyclic coordinate descent is used, extrapolation improves the performance of screening rules, and that using a dual-based working set policy further reduces the computational burden.

\begin{figure}[htbp]
    \centering
    \includegraphics[width=0.7\linewidth]{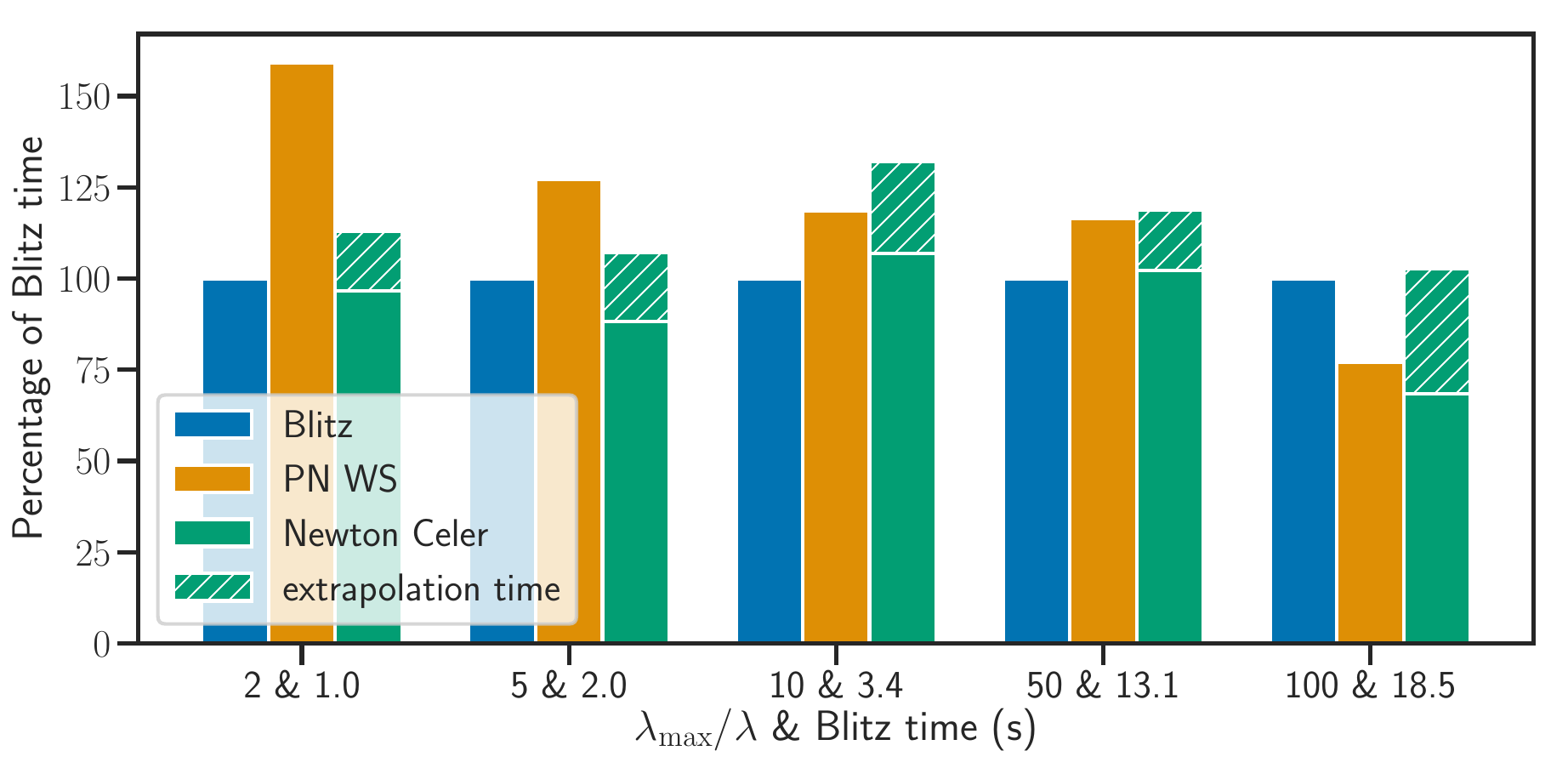}
    \caption{Time to solve a Logistic regression problem for different values of $\lambda$, on the \emph{rcv1 dataset} ($\epsilon = 10^{-6}$).}
    \label{fig:single_alpha_rcv1_logreg}
\end{figure}

\Cref{fig:single_alpha_rcv1_logreg} shows the limitation of dual extrapolation when second order information is taken into account with a Prox-Newton:
because the Prox-Newton iterations do not create a VAR sequence, it is necessary to perform some passes of coordinate descent to create $\thetaccel$, as detailed in \Cref{sub:newton_celer}.
This particular experiment reveals that this additional time unfortunately mitigates the gains observed in better working sets and stopping criterion.

\subsubsection{Multitask Lasso}

The data for this experiment uses magnetoencephalography (MEG) recordings which are
collected for neuroscience studies. Here we use data from the \emph{sample} dataset of
the MNE software \citep{mne}. Data were obtained using auditory stimulation.
There are $n = 305$ sensors, $p = \num{7498}$ source locations in the brain, and the measurements are time series of length $q = 49$.
Using a Multitask Lasso formulation allows to reconstruct brain activitiy exhibiting a stable sparsity pattern across time \citep{Gramfort_Kowalski_Hamalainen12}.
The inner solver for \celer is block coordinate descent, which is also used for the Gap Safe solver \citep{Ndiaye_Fercoq_Gramfort_Salmon15}.

\begin{figure}[htbp]
    \centering
    \includegraphics[width=0.75\linewidth]{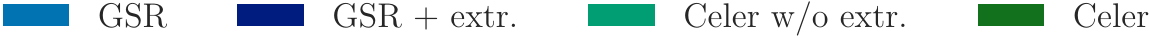} \\
    \includegraphics[width=0.45\linewidth]{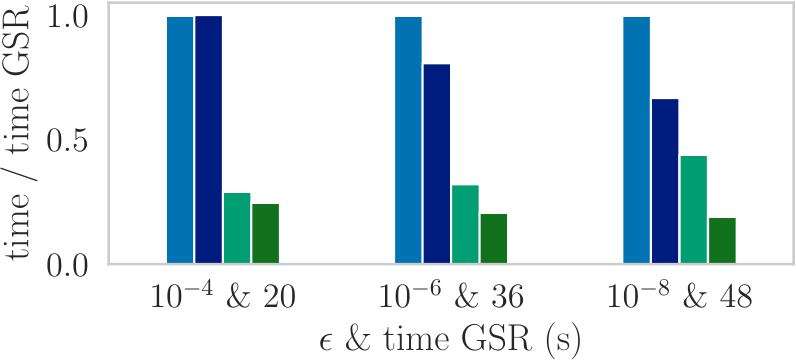}
    \includegraphics[width=0.45\linewidth]{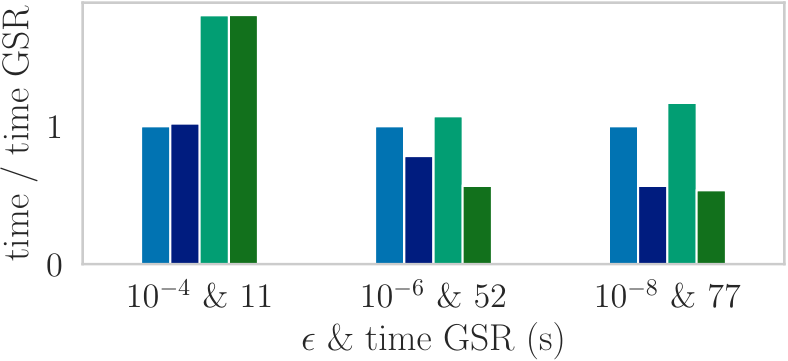}
    \caption{Time to compute a Multitask Lasso path from $\lambda_{\max}$ to $\lambda_{\max} / 100$ on MEG data (\textbf{left}: coarse grid of 10 values, \textbf{right}: fine grid of 100 values).
    $\lambda_{\max} / 100$ gives \num{254} non-zero rows for $\hat \Beta$.}
    \label{fig:mtl_path}
\end{figure}

While \Cref{fig:VAR_MTL} showed that for the Multitask Lasso the dual extrapolation performance also gives an improved duality gap, here \Cref{fig:mtl_path} shows that the working set policy of \celer performs better than Gap Safes rules with strong active warm start.
We could not include \blitz in the benchmark as there is no standard public implementation for this problem.

\section*{Conclusion}
In this work, we generalize the dual extrapolation procedure for the Lasso (\celer) to a wider class of $\ell_1$-penalized problems, in particular sparse Logistic regression.
Theoretical guarantees based on \emph{sign identification} of coordinate descent are provided.
Experiments show that dual extrapolation yields more efficient Gap Safe screening rules and working sets solver.
Finally, we adapt \celer to make it compatible with prox-Newton solvers, and empirically demonstrate its applicability to the Multi-task Lasso, for which we leave the proof to future work.

\acks{This work was funded by the ERC Starting Grant SLAB ERC-StG-676943, by the Chair Machine Learning for Big Data at T\'el\'ecom ParisTech and by the ANR grant GraVa ANR-18-CE40-0005.}

\vskip 0.2in
\bibliography{references_all}

\providecommand{\AC}{A.-C}\providecommand{\CA}{C.-A}\providecommand{\CH}{C.-H}\providecommand{\CN}{C.-N}\providecommand{\CJ}{C.-J}\providecommand{\HJ}{H.-J}\providecommand{\HY}{H.-Y}\providecommand{\JC}{J.-C}\providecommand{\JP}{J.-P}\providecommand{\JB}{J.-B}\providecommand{\JF}{J.-F}\providecommand{\JJ}{J.-J}\providecommand{\JM}{J.-M}\providecommand{\KW}{K.-W}\providecommand{\KR}{K.-R}\providecommand{\PL}{P.-L}\providecommand{\RE}{R.-E}\providecommand{\SH}{S.-H}\providecommand{\SJ}{S.-J}\providecommand{\XR}{X.-R}\providecommand{\WX}{W.-X}\providecommand{\PL}{P.-L}\providecommand{\YX}{Y.-X}
\begin{thebibliography}{69}
\providecommand{\natexlab}[1]{#1}
\providecommand{\url}[1]{\texttt{#1}}
\expandafter\ifx\csname urlstyle\endcsname\relax
  \providecommand{\doi}[1]{doi: #1}\else
  \providecommand{\doi}{doi: \begingroup \urlstyle{rm}\Url}\fi

\bibitem[Arrow et~al.(1958)Arrow, Hurwicz, and Uzawa]{Arrow_Hurwicz_Uzawa58}
K.~Arrow, L.~Hurwicz, and H.~Uzawa.
\newblock \emph{Studies in Nonlinear Programming}.
\newblock {Stanford University Press}, 1958.

\bibitem[Bach et~al.(2012)Bach, Jenatton, Mairal, and
  Obozinski]{Bach_Jenatton_Mairal_Obozinski12}
F.~Bach, R.~Jenatton, J.~Mairal, and G.~Obozinski.
\newblock Convex optimization with sparsity-inducing norms.
\newblock \emph{Foundations and Trends in Machine Learning}, 4\penalty0
  (1):\penalty0 1--106, 2012.

\bibitem[Bauschke and Combettes(2011)]{Bauschke_Combettes11}
H.~H. Bauschke and P.~L. Combettes.
\newblock \emph{Convex analysis and monotone operator theory in {H}ilbert
  spaces}.
\newblock Springer, New York, 2011.

\bibitem[Behnel et~al.(2011)Behnel, Bradshaw, Citro, Dalcin, Seljebotn, and
  Smith]{Behnel_etal11}
S.~Behnel, R.~Bradshaw, C.~Citro, L.~Dalcin, D.~S. Seljebotn, and K.~Smith.
\newblock Cython: The best of both worlds.
\newblock \emph{Computing in Science Engineering}, 13\penalty0 (2):\penalty0 31
  --39, 2011.

\bibitem[Belloni et~al.(2011)Belloni, Chernozhukov, and
  Wang]{Belloni_Chernozhukov_Wang11}
A.~Belloni, V.~Chernozhukov, and L.~Wang.
\newblock Square-root {Lasso}: pivotal recovery of sparse signals via conic
  programming.
\newblock \emph{Biometrika}, 98\penalty0 (4):\penalty0 791--806, 2011.

\bibitem[Boisbunon et~al.(2014)Boisbunon, Flamary, and
  Rakotomamonjy]{Boisbunon_Flamary_Rakoto14}
A.~Boisbunon, R.~Flamary, and A.~Rakotomamonjy.
\newblock Active set strategy for high-dimensional non-convex sparse
  optimization problems.
\newblock In \emph{ICASSP}, pages 1517--1521, 2014.

\bibitem[Bonnefoy et~al.(2014)Bonnefoy, Emiya, Ralaivola, and
  Gribonval]{Bonnefoy_Emiya_Ralaivola_Gribonval14}
A.~Bonnefoy, V.~Emiya, L.~Ralaivola, and R.~Gribonval.
\newblock A dynamic screening principle for the lasso.
\newblock In \emph{EUSIPCO}, 2014.

\bibitem[Boyd and Vandenberghe(2004)]{Boyd_Vandenberghe04}
S.~Boyd and L.~Vandenberghe.
\newblock \emph{Convex optimization}.
\newblock Cambridge University Press, 2004.

\bibitem[{Buitinck} et~al.(2013){Buitinck}, {Louppe}, {Blondel}, {Pedregosa},
  {Mueller}, {Grisel}, {Niculae}, {Prettenhofer}, {Gramfort}, {Grobler},
  {Layton}, {Vanderplas}, {Joly}, {Holt}, and {Varoquaux}]{sklearn_api}
L.~{Buitinck}, G.~{Louppe}, M.~{Blondel}, F.~{Pedregosa}, A.~{Mueller},
  O.~{Grisel}, V.~{Niculae}, P.~{Prettenhofer}, A.~{Gramfort}, J.~{Grobler},
  R.~{Layton}, J.~{Vanderplas}, A.~{Joly}, B.~{Holt}, and G.~{Varoquaux}.
\newblock {API design for machine learning software: experiences from the
  scikit-learn project}.
\newblock \emph{arXiv e-prints}, 2013.

\bibitem[Cand\`es and Recht(2013)]{candes2013simple}
E.~Cand\`es and B.~Recht.
\newblock Simple bounds for recovering low-complexity models.
\newblock \emph{Mathematical Programming}, 141\penalty0 (1-2):\penalty0
  577--589, 2013.

\bibitem[Chambolle and Pock(2011)]{Chambolle_Pock11}
A.~Chambolle and T.~Pock.
\newblock A first-order primal-dual algorithm for convex problems with
  applications to imaging.
\newblock \emph{J. Math. Imaging Vis.}, 40\penalty0 (1):\penalty0 120--145,
  2011.

\bibitem[Chen and Donoho(1995)]{Chen_Donoho95}
S.~S. Chen and D.~L. Donoho.
\newblock Atomic decomposition by basis pursuit.
\newblock In \emph{SPIE}, 1995.

\bibitem[Diamond and Boyd(2016)]{cvxpy}
S.~Diamond and S.~Boyd.
\newblock {CVXPY}: A {P}ython-embedded modeling language for convex
  optimization.
\newblock \emph{J. Mach. Learn. Res.}, 17\penalty0 (83):\penalty0 1--5, 2016.

\bibitem[D{\"{u}}nner et~al.(2016)D{\"{u}}nner, S.Forte, Tak{\'{a}\v{c}}, and
  Jaggi]{Dunner_Forte_Takac_Jaggi16}
C.~D{\"{u}}nner, S.Forte, M.~Tak{\'{a}\v{c}}, and M.~Jaggi.
\newblock Primal-dual rates and certificates.
\newblock In \emph{ICML}, pages 783--792, 2016.

\bibitem[{El Ghaoui} et~al.(2012){El Ghaoui}, Viallon, and
  Rabbani]{ElGhaoui_Viallon_Rabbani12}
L.~{El Ghaoui}, V.~Viallon, and T.~Rabbani.
\newblock Safe feature elimination in sparse supervised learning.
\newblock \emph{J. Pacific Optim.}, 8\penalty0 (4):\penalty0 667--698, 2012.

\bibitem[Fan et~al.(2008)Fan, Chang, Hsieh, Wang, and
  Lin]{Fan_Chang_Hsieh_Wang_Lin08}
{\RE}.~Fan, {\KW}.~Chang, {\CJ}.~Hsieh, {\XR}.~Wang, and {\CJ}.~Lin.
\newblock Liblinear: A library for large linear classification.
\newblock \emph{J. Mach. Learn. Res.}, 9:\penalty0 1871--1874, 2008.

\bibitem[Fan and Lv(2008)]{Fan_Lv2008}
J.~Fan and J.~Lv.
\newblock Sure independence screening for ultrahigh dimensional feature space.
\newblock \emph{J. R. Stat. Soc. Ser. B Stat. Methodol.}, 70\penalty0
  (5):\penalty0 849--911, 2008.

\bibitem[Fercoq and Bianchi(2015)]{Fercoq_Bianchi15}
O.~Fercoq and P.~Bianchi.
\newblock A coordinate descent primal-dual algorithm with large step size and
  possibly non separable functions.
\newblock \emph{arXiv preprint arXiv:1508.04625}, 2015.

\bibitem[Fercoq and Richt\'{a}rik(2015)]{Fercoq_Richtarik15}
O.~Fercoq and P.~Richt\'{a}rik.
\newblock Accelerated, parallel and proximal coordinate descent.
\newblock \emph{SIAM J. Optim.}, 25\penalty0 (3):\penalty0 1997 -- 2013, 2015.

\bibitem[Fercoq et~al.(2015)Fercoq, Gramfort, and
  Salmon]{Fercoq_Gramfort_Salmon15}
O.~Fercoq, A.~Gramfort, and J.~Salmon.
\newblock Mind the duality gap: safer rules for the lasso.
\newblock In \emph{ICML}, pages 333--342, 2015.

\bibitem[Friedman et~al.(2007)Friedman, Hastie, H{\"o}fling, and
  Tibshirani]{Friedman_Hastie_Hofling_Tibshirani07}
J.~Friedman, T.~J. Hastie, H.~H{\"o}fling, and R.~Tibshirani.
\newblock Pathwise coordinate optimization.
\newblock \emph{Ann. Appl. Stat.}, 1\penalty0 (2):\penalty0 302--332, 2007.

\bibitem[Friedman et~al.(2010)Friedman, Hastie, and
  Tibshirani]{Friedman_Hastie_Tibshirani10}
J.~Friedman, T.~J. Hastie, and R.~Tibshirani.
\newblock Regularization paths for generalized linear models via coordinate
  descent.
\newblock \emph{J. Stat. Softw.}, 33\penalty0 (1):\penalty0 1, 2010.

\bibitem[Fuchs(2004)]{fuchs2004sparse}
J.-J. Fuchs.
\newblock On sparse representations in arbitrary redundant bases.
\newblock \emph{IEEE transactions on Information theory}, 50\penalty0
  (6):\penalty0 1341--1344, 2004.

\bibitem[Gramfort et~al.(2012)Gramfort, Kowalski, and
  H{\"a}m{\"a}l{\"a}inen]{Gramfort_Kowalski_Hamalainen12}
A.~Gramfort, M.~Kowalski, and M.~H{\"a}m{\"a}l{\"a}inen.
\newblock Mixed-norm estimates for the {M/EEG} inverse problem using
  accelerated gradient methods.
\newblock \emph{Phys. Med. Biol.}, 57\penalty0 (7):\penalty0 1937--1961, 2012.

\bibitem[Gramfort et~al.(2014)Gramfort, Luessi, Larson, Engemann, Strohmeier,
  Brodbeck, Parkkonen, and H{\"a}m{\"a}l{\"a}inen]{mne}
A.~Gramfort, M.~Luessi, E.~Larson, D.~A. Engemann, D.~Strohmeier, C.~Brodbeck,
  L.~Parkkonen, and M.~S. H{\"a}m{\"a}l{\"a}inen.
\newblock {MNE} software for processing {MEG} and {EEG} data.
\newblock \emph{NeuroImage}, 86:\penalty0 446 -- 460, 2014.

\bibitem[Hale et~al.(2008)Hale, Yin, and Zhang]{Hale_Yin_Zhang2008}
E.~Hale, W.~Yin, and Y.~Zhang.
\newblock Fixed-point continuation for $\ell_1$-minimization: Methodology and
  convergence.
\newblock \emph{SIAM J. Optim.}, 19\penalty0 (3):\penalty0 1107--1130, 2008.

\bibitem[Hare and Lewis(2007)]{hare2007identifying}
W.~L. Hare and A.~S. Lewis.
\newblock Identifying active manifolds.
\newblock \emph{Algorithmic Operations Research}, 2\penalty0 (2):\penalty0
  75--75, 2007.

\bibitem[Hiriart-Urruty and Lemar{\'e}chal(1993)]{Hiriart-Urruty_Lemarechal93b}
{\JB}.~Hiriart-Urruty and C.~Lemar{\'e}chal.
\newblock \emph{Convex analysis and minimization algorithms. {II}}, volume 306.
\newblock Springer-Verlag, Berlin, 1993.

\bibitem[Hsieh et~al.(2014)Hsieh, Sustik, Dhillon, and
  Ravikumar]{Hsieh_Sustik_Dhillon_Ravikumar14}
{\CJ}~Hsieh, M.~Sustik, I.~Dhillon, and P.~Ravikumar.
\newblock {QUIC}: Quadratic approximation for sparse inverse covariance
  estimation.
\newblock \emph{J. Mach. Learn. Res.}, 15:\penalty0 2911--2947, 2014.

\bibitem[Johnson and Guestrin(2015)]{Johnson_Guestrin15}
T.~B. Johnson and C.~Guestrin.
\newblock Blitz: A principled meta-algorithm for scaling sparse optimization.
\newblock In \emph{ICML}, pages 1171--1179, 2015.

\bibitem[Johnson and Guestrin(2017)]{Johnson_Guestrin17}
T.~B. Johnson and C.~Guestrin.
\newblock {S}tingy{CD}: Safely avoiding wasteful updates in coordinate descent.
\newblock In \emph{ICML}, pages 1752--1760, 2017.

\bibitem[Johnson and Guestrin(2018)]{Johnson_Guestrin18}
T.~B. Johnson and C.~Guestrin.
\newblock A fast, principled working set algorithm for exploiting piecewise
  linear structure in convex problems.
\newblock \emph{arXiv preprint arXiv:1807.08046}, 2018.

\bibitem[Karimireddy et~al.(2018)Karimireddy, Koloskova, Stich, and
  Jaggi]{Karimireddy_Koloskova_Stich_Jaggi18}
P.~Karimireddy, A.~Koloskova, S.~Stich, and M.~Jaggi.
\newblock {Efficient Greedy Coordinate Descent for Composite Problems}.
\newblock \emph{arXiv preprint arXiv:1810.06999}, 2018.

\bibitem[Koh et~al.(2007)Koh, Kim, and Boyd]{Koh_Kim_Boyd07}
K.~Koh, {\SJ}.~Kim, and S.~Boyd.
\newblock An interior-point method for large-scale l1-regularized logistic
  regression.
\newblock \emph{J. Mach. Learn. Res.}, 8\penalty0 (8):\penalty0 1519--1555,
  2007.

\bibitem[Kowalski et~al.(2011)Kowalski, Weiss, Gramfort, and
  Anthoine]{Kowalski_Weiss_Gramfort_Anthoine11}
M.~Kowalski, P.~Weiss, A.~Gramfort, and S.~Anthoine.
\newblock Accelerating {ISTA} with an active set strategy.
\newblock In \emph{OPT 2011: 4th International Workshop on Optimization for
  Machine Learning}, page~7, 2011.

\bibitem[Lam et~al.(2015)Lam, Pitrou, and Seibert]{Lam_Pitrou_Seibert15}
S.~K. Lam, A.~Pitrou, and S.~Seibert.
\newblock {Numba: A LLVM-based Python JIT Compiler}.
\newblock In \emph{Proceedings of the Second Workshop on the LLVM Compiler
  Infrastructure in HPC}, pages 1--6. ACM, 2015.

\bibitem[Lee et~al.(2012)Lee, Sun, and Saunders]{Lee_Sun_Saunders12}
J.~Lee, Y.~Sun, and M.~Saunders.
\newblock Proximal {N}ewton-type methods for convex optimization.
\newblock In \emph{NIPS}, pages 827--835, 2012.

\bibitem[Mairal(2010)]{Mairal}
J.~Mairal.
\newblock \emph{Sparse coding for machine learning, image processing and
  computer vision}.
\newblock PhD thesis, {\'E}cole normale sup{\'e}rieure de Cachan, 2010.

\bibitem[Massias et~al.(2017)Massias, Gramfort, and
  Salmon]{Massias_Gramfort_Salmon17}
M.~Massias, A.~Gramfort, and J.~Salmon.
\newblock From safe screening rules to working sets for faster lasso-type
  solvers.
\newblock In \emph{10th NIPS Workshop on Optimization for Machine Learning},
  2017.

\bibitem[Massias et~al.(2018)Massias, Gramfort, and
  Salmon]{Massias_Gramfort_Salmon18}
M.~Massias, A.~Gramfort, and J.~Salmon.
\newblock {C}eler: a fast solver for the {L}asso with dual extrapolation.
\newblock In \emph{ICML}, 2018.

\bibitem[McCullagh and Nelder(1989)]{Mccullagh_Nelder1989}
P.~McCullagh and J.A. Nelder.
\newblock \emph{Generalized Linear Models, Second Edition}.
\newblock Chapman and Hall/CRC Monographs on Statistics and Applied Probability
  Series. 1989.

\bibitem[Myers and Shih(1988)]{Myers_Shih88}
D.~Myers and W.~Shih.
\newblock A constraint selection technique for a class of linear programs.
\newblock \emph{Operations Research Letters}, 7\penalty0 (4):\penalty0
  191--195, 1988.

\bibitem[Ndiaye et~al.(2015)Ndiaye, Fercoq, Gramfort, and
  Salmon]{Ndiaye_Fercoq_Gramfort_Salmon15}
E.~Ndiaye, O.~Fercoq, A.~Gramfort, and J.~Salmon.
\newblock Gap safe screening rules for sparse multi-task and multi-class
  models.
\newblock In \emph{NIPS}, pages 811--819, 2015.

\bibitem[Ndiaye et~al.(2016)Ndiaye, Fercoq, Gramfort, and
  Salmon]{Ndiaye_Fercoq_Gramfort_Salmon16}
E.~Ndiaye, O.~Fercoq, A.~Gramfort, and J.~Salmon.
\newblock {GAP} safe screening rules for sparse-group-lasso.
\newblock In \emph{NIPS}, 2016.

\bibitem[Ndiaye et~al.(2017)Ndiaye, Fercoq, Gramfort, and
  Salmon]{Ndiaye_Fercoq_Gramfort_Salmon16b}
E.~Ndiaye, O.~Fercoq, A.~Gramfort, and J.~Salmon.
\newblock Gap safe screening rules for sparsity enforcing penalties.
\newblock \emph{J. Mach. Learn. Res.}, 18\penalty0 (128):\penalty0 1--33, 2017.

\bibitem[Obozinski et~al.(2010)Obozinski, Taskar, and
  Jordan]{Obozinski_Taskar_Jordan10}
G.~Obozinski, B.~Taskar, and M.~I. Jordan.
\newblock Joint covariate selection and joint subspace selection for multiple
  classification problems.
\newblock \emph{Statistics and Computing}, 20\penalty0 (2):\penalty0 231--252,
  2010.

\bibitem[Ogawa et~al.(2013)Ogawa, Suzuki, and
  Takeuchi]{Ogawa_Suzuki_Takeuchi13}
K.~Ogawa, Y.~Suzuki, and I.~Takeuchi.
\newblock Safe screening of non-support vectors in pathwise {SVM} computation.
\newblock In \emph{ICML}, pages 1382--1390, 2013.

\bibitem[Palacios-Gomez et~al.(1982)Palacios-Gomez, Lasdon, and
  Engquist]{Palacios_Lasdon_Engquist82}
F.~Palacios-Gomez, L.~Lasdon, and M.~Engquist.
\newblock Nonlinear optimization by successive linear programming.
\newblock \emph{Management Science}, 28\penalty0 (10):\penalty0 1106--1120,
  1982.

\bibitem[Pedregosa et~al.(2011)Pedregosa, Varoquaux, Gramfort, Michel, Thirion,
  Grisel, Blondel, Prettenhofer, Weiss, Dubourg, Vanderplas, Passos,
  Cournapeau, Brucher, Perrot, and Duchesnay]{Pedregosa_etal11}
F.~Pedregosa, G.~Varoquaux, A.~Gramfort, V.~Michel, B.~Thirion, O.~Grisel,
  M.~Blondel, P.~Prettenhofer, R.~Weiss, V.~Dubourg, J.~Vanderplas, A.~Passos,
  D.~Cournapeau, M.~Brucher, M.~Perrot, and E.~Duchesnay.
\newblock Scikit-learn: Machine learning in {P}ython.
\newblock \emph{J. Mach. Learn. Res.}, 12:\penalty0 2825--2830, 2011.

\bibitem[Perekrestenko et~al.(2017)Perekrestenko, Cevher, and
  Jaggi]{Perekrestenko_Cevher_Jaggi17}
D.~Perekrestenko, V.~Cevher, and M.~Jaggi.
\newblock Faster coordinate descent via adaptive importance sampling.
\newblock In \emph{AISTATS}, pages 869--877, 2017.

\bibitem[Richt{\'a}rik and Tak{\'a}{\v{c}}(2014)]{Richtarik_Takac14}
P.~Richt{\'a}rik and M.~Tak{\'a}{\v{c}}.
\newblock Iteration complexity of randomized block-coordinate descent methods
  for minimizing a composite function.
\newblock \emph{Mathematical Programming}, 144\penalty0 (1-2):\penalty0 1--38,
  2014.

\bibitem[Roth and Fischer(2008)]{Roth_Fischer08}
V.~Roth and B.~Fischer.
\newblock The group-lasso for generalized linear models: uniqueness of
  solutions and efficient algorithms.
\newblock In \emph{ICML}, pages 848--855, 2008.

\bibitem[Santis et~al.(2016)Santis, Lucidi, and
  Rinaldi]{DeSantis_Lucidi_Rinaldi16}
M.~De Santis, S.~Lucidi, and F.~Rinaldi.
\newblock A fast active set block coordinate descent algorithm for
  $\ell_1$-regularized least squares.
\newblock \emph{SIAM J. Optim.}, 26\penalty0 (1):\penalty0 781--809, 2016.

\bibitem[Scheinberg and Tang(2013)]{Scheinberg_Tang13}
K.~Scheinberg and X.~Tang.
\newblock Complexity of inexact proximal {N}ewton methods.
\newblock \emph{arXiv preprint arxiv:1311.6547}, 2013.

\bibitem[Scieur(2018)]{Scieur}
D.~Scieur.
\newblock \emph{Acceleration in Optimization}.
\newblock PhD thesis, {\'E}cole normale sup{\'e}rieure, 2018.

\bibitem[Scieur et~al.(2016)Scieur, d'Aspremont, and
  Bach]{Scieur_Bach_Daspremont16}
D.~Scieur, A.~d'Aspremont, and F.~Bach.
\newblock Regularized nonlinear acceleration.
\newblock In \emph{NIPS}, pages 712--720, 2016.

\bibitem[Simon et~al.(2013)Simon, Friedman, Hastie, and
  Tibshirani]{Simon_Friedman_Hastie_Tibshirani12}
N.~Simon, J.~Friedman, T.~J. Hastie, and R.~Tibshirani.
\newblock A sparse-group lasso.
\newblock \emph{J. Comput. Graph. Statist.}, 22\penalty0 (2):\penalty0
  231--245, 2013.
\newblock ISSN 1061-8600.

\bibitem[Thompson et~al.(1966)Thompson, Tonge, and
  Zionts]{Thompson_Tonge_Zionts66}
G.~Thompson, F.~Tonge, and S.~Zionts.
\newblock Techniques for removing nonbinding constraints and extraneous
  variables from linear programming problems.
\newblock \emph{Management Science}, 12\penalty0 (7):\penalty0 588--608, 1966.

\bibitem[Tibshirani(1996)]{Tibshirani96}
R.~Tibshirani.
\newblock Regression shrinkage and selection via the lasso.
\newblock \emph{J. R. Stat. Soc. Ser. B Stat. Methodol.}, 58\penalty0
  (1):\penalty0 267--288, 1996.

\bibitem[Tibshirani et~al.(2012)Tibshirani, Bien, Friedman, Hastie, Simon,
  Taylor, and Tibshirani]{Tibshirani_Bien_Friedman_Hastie_Simon_Tibshirani12}
R.~Tibshirani, J.~Bien, J.~Friedman, T.~J. Hastie, N.~Simon, J.~Taylor, and
  R.~J. Tibshirani.
\newblock Strong rules for discarding predictors in lasso-type problems.
\newblock \emph{J. R. Stat. Soc. Ser. B Stat. Methodol.}, 74\penalty0
  (2):\penalty0 245--266, 2012.

\bibitem[Tibshirani(2013)]{Tibshirani13}
R.~J. Tibshirani.
\newblock The lasso problem and uniqueness.
\newblock \emph{Electron. J. Stat.}, 7:\penalty0 1456--1490, 2013.

\bibitem[Tibshirani(2017)]{Tibshirani17}
R.~J. Tibshirani.
\newblock {Dykstra's Algorithm, ADMM, and Coordinate Descent: Connections,
  Insights, and Extensions}.
\newblock In \emph{NIPS}, pages 517--528, 2017.

\bibitem[Tseng(2001)]{Tseng01}
P.~Tseng.
\newblock Convergence of a block coordinate descent method for
  nondifferentiable minimization.
\newblock \emph{J. Optim. Theory Appl.}, 109\penalty0 (3):\penalty0 475--494,
  2001.

\bibitem[Vaiter et~al.(2015)Vaiter, Golbabaee, Fadili, and
  Peyr{\'e}]{vaiter2015model}
S.~Vaiter, M.~Golbabaee, J.~Fadili, and G.~Peyr{\'e}.
\newblock Model selection with low complexity priors.
\newblock \emph{Information and Inference: A Journal of the IMA}, 4\penalty0
  (3):\penalty0 230--287, 2015.

\bibitem[Vaiter et~al.(2018)Vaiter, Peyr\'e, and Fadili]{vaiter2017model}
S.~Vaiter, G.~Peyr\'e, and J.~M. Fadili.
\newblock Model consistency of partly smooth regularizers.
\newblock \emph{{IEEE} Trans. Inf. Theory}, 64\penalty0 (3):\penalty0
  1725--1737, 2018.

\bibitem[Wang et~al.(2012)Wang, Wonka, and Ye]{Wang_Wonka_Ye12}
J.~Wang, P.~Wonka, and J.~Ye.
\newblock Lasso screening rules via dual polytope projection.
\newblock \emph{arXiv preprint arXiv:1211.3966}, 2012.

\bibitem[Xiang et~al.(2016)Xiang, Wang, and Ramadge]{Xiang_Wang_Ramadge14}
Z.~J. Xiang, Y.~Wang, and P.~J. Ramadge.
\newblock Screening tests for lasso problems.
\newblock \emph{{IEEE} Trans. Pattern Anal. Mach. Intell.}, PP\penalty0 (99),
  2016.

\bibitem[Yuan et~al.(2012)Yuan, Ho, and Lin]{Yuan_Ho_Lin12}
G~Yuan, {\CH}~Ho, and {\CJ}~Lin.
\newblock An improved {GLMNET} for l1-regularized logistic regression.
\newblock \emph{J. Mach. Learn. Res.}, 13:\penalty0 1999--2030, 2012.

\bibitem[Yuan and Lin(2006)]{Yuan_Lin06}
M.~Yuan and Y.~Lin.
\newblock Model selection and estimation in regression with grouped variables.
\newblock \emph{J. R. Stat. Soc. Ser. B Stat. Methodol.}, 68\penalty0
  (1):\penalty0 49--67, 2006.

\end{thebibliography}

\end{document}